\documentclass{article}

\usepackage[final]{neurips_2024}

\usepackage[utf8]{inputenc} 
\usepackage[T1]{fontenc}    
\usepackage{hyperref}       
\usepackage{url}            
\usepackage{booktabs}       
\usepackage{amsfonts}       
\usepackage{nicefrac}       
\usepackage{microtype}      
\usepackage{xcolor}         
\usepackage{verbatim}

\usepackage{amsmath}
\usepackage{amsthm}
\usepackage{amssymb}

\usepackage{graphicx,amscd}
\usepackage{mathrsfs}
\usepackage{enumerate}
\usepackage{mathrsfs}
\usepackage{dsfont}
\usepackage{stmaryrd}
\usepackage{wrapfig}
\usepackage{pstricks}
\usepackage{subcaption}
\usepackage{float}
\usepackage{physics}

\makeatletter
\theoremstyle{plain}
\newtheorem{thm}{\protect\theoremname}
\theoremstyle{remark}
\newtheorem*{rem}{\protect\remarkname}

\theoremstyle{plain}

\newtheorem{theorem}{Theorem}[section]
\newtheorem*{theorem*}{Theorem}
\newtheorem{definition}[theorem]{Definition}
\newtheorem{proposition}[theorem]{Proposition}

\newtheorem*{definition*}{Definition}
\newtheorem*{counterex*}{Counterexample}
\newtheorem*{example*}{Example}

\newtheorem*{remark}{Remark}
\newtheorem{lemma}[theorem]{Lemma}

\newcommand{\abra}[1]{\langle #1 \rangle} 
\usepackage[shortlabels]{enumitem} 

\renewcommand{\max}{\text{max}}
\renewcommand{\min}{\text{min}}

\title{Mixed Dynamics In Linear Networks:\\ Unifying the Lazy and Active Regimes}

\author{
  Zhenfeng Tu \\
  Courant Institute\\
  New York University\\
  New York, NY 10012 \\
  \texttt{zt2255@nyu.edu} \\
  \And
  Santiago Aranguri \\
  Courant Institute\\
  New York University\\
  New York, NY 10012 \\
  \texttt{aranguri@nyu.edu} \\
  \And
  Arthur Jacot \\
  Courant Institute\\
  New York University\\
  New York, NY 10012 \\
  \texttt{arthur.jacot@nyu.edu} \\
}

\makeatother

\usepackage{babel}
\providecommand{\remarkname}{Remark}
\providecommand{\theoremname}{Theorem}

\begin{document}

\maketitle

\begin{abstract}
The training dynamics of linear networks are well studied in two distinct
setups: the lazy regime and balanced/active regime, depending on the
initialization and width of the network. We provide a surprisingly
simple unifying formula for the evolution of the learned matrix that
contains as special cases both lazy and balanced regimes but also
a mixed regime in between the two. In the mixed regime, a part of
the network is lazy while the other is balanced. More precisely the
network is lazy along singular values that are below a certain threshold
and balanced along those that are above the same threshold. At initialization,
all singular values are lazy, allowing for the network to align itself
with the task, so that later in time, when some of the singular value
cross the threshold and become active they will converge rapidly (convergence
in the balanced regime is notoriously difficult in the absence of
alignment). The mixed regime is the `best of both worlds': it converges
from any random initialization (in contrast to balanced dynamics which
require special initialization), and has a low rank bias (absent in
the lazy dynamics). This allows us to prove an almost complete phase
diagram of training behavior as a function of the variance at initialization
and the width, for a MSE training task.
\end{abstract}
\maketitle

\section{Introduction}

Whether in linear networks or nonlinear ones, there has been a lot
of interest in the distinction between the lazy regime \cite{jacot2018neural}
and the active regime \cite{Chizat2018,Rotskoff2018,chizat2018note,yang2020feature,bordelon2022deep_MF}
as the number of neurons grows towards infinity. In the lazy regime
the training dynamics become linear, so that they can be easily described
in terms of the Neural Tangent Kernel (NTK) \cite{jacot2018neural,exact_arora2019,yang2020NTK,liu2020_NTK-PL-inequ},
while the active regime exhibits complex nonlinear dynamics. While
our understanding of the active regime remains much more limited,
it appears to be characterized by the emergence of feature learning\cite{geiger2019disentangling,yang2020feature},
and of a form of sparsity \cite{advani_2017_independent_diag,bach2017_F1_norm,abbe_2021_staircase,abbe_2022_staircase}
(the type of sparsity observed depends on the network type \cite{bach2017_F1_norm,dai_2021_repres_cost_DLN,jacot_2022_BN_rank_short,jacot_2023_bottleneck2},
but we will focus on fully-connected linear networks which exhibit
a rank sparsity in the learned linear map \cite{arora_2018_depth_speed_Lp,li2020towards,jacot-2021-DLN-Saddle,wang_2023_bias_SGD_L2}) 
which are both absent in the lazy regime .

Note that even though it is common to talk of the `the' active regime,
we do not know yet whether there is only one or multiple active regimes.
Indeed the term active regime is usually used to describe any regime
that differs from the lazy regime and exhibit some form of feature
learning. Though we do not have an complete understanding of where
the lazy regimes ends and the active regime(s) start, we know that
the lazy regime requires extreme overparametrization (a large number
of neurons in comparison to the number of datapoints) \cite{Allen-Zhu2018,Du2019},
a `large' initialization of the weights \cite{chizat2018note}, a
small learning rate, and early stopping when using a cross-entropy
loss or weight decay . Indeed, active regimes have been observed
by breaking either of these requirements: taking limits with mild
or no overparametrization \cite{arous2022summary_stats},
taking smaller or even vanishingly small initializations \cite{li2020towards,jacot-2021-DLN-Saddle}, using large learning rates \cite{lewkowycz_2020_large_lr} or SGD \cite{Pesme_2021_SGD_bias_diag_nets,wang_2023_bias_SGD_L2},
or studying the late training dynamics with the cross-entropy loss
\cite{Ji_2018_directional,chizat_2020_implicit_bias} or weight decay
\cite{lewkowycz2020_L2_training,Ongie_2020_repres_bounded_norm_shallow_ReLU_net,jacot_2022_L2_reformulation,jacot_2022_BN_rank}.
Though each of these can lead to active regimes with significantly
different dynamics, they often lead to similar types of feature
learning and sparsity.

In this paper, we study this transition in the context of linear networks
and focus mainly on the effects of the width $w$ and the variance
of the weights at initialization $\sigma^{2}$, and give a precise
and almost complete phase diagram, showing the transitions between
lazy and active regimes. In this setting, we will show that there
typically is only `one' active regime, which is the same (up to approximation)
as the already well-studied balanced regime \cite{saxe_2014_exact,arora_2018_depth_speed_Lp,arora_2019_matrix_factorization}. 

But our result also paint a more subtle picture than the lazy/active
dichotomy. We propose a more granular approach, where at a certain
time some part of the network can be in the lazy regime, while others
are in the active or balanced regime. More precisely the network is
lazy along the singular values of the matrix represented by the network
that are smaller than $\sigma^{2}w$, and in the active regime along
the singular values larger than $\sigma^{2}w$.

\subsection{Contributions}

We consider the training dynamics of shallow linear networks $A_{\theta}=W_{2}W_{1}$
and show that for large enough width $w$ (the inner dimension), and
a iid $\mathcal{N}(0,\sigma^{2})$ initialization of all weights,
the dynamics of $A_{\theta(t)}$ as a result of training the parameters
$\theta=(W_{1},W_{2})$ with GD/GF on the loss $\mathcal{L}(\theta)=C(A_{\theta})$
for a general matrix cost $C$ with learning rate $\eta$ is approximately given by the self-consistent dynamics
\begin{equation}
    \partial_{t}A_{\theta(t)}\approx-\eta \sqrt{A_{\theta}A_{\theta}^{T}+\sigma^{4}w^{2}I}\nabla C(A_{\theta})-\eta\nabla C(A_{\theta})\sqrt{A_{\theta}^{T}A_{\theta}+\sigma^{4}w^{2}I}.
    \label{eq:self-consistent}
\end{equation}

These dynamics contain as special cases both the lazy dynamics 
\[
\partial_{t}A_{\theta(t)}\approx-2\eta\sigma^{2}w\nabla C(A_{\theta})
\]
 when $\sigma^{2}w\gg\lambda_{max}(A_{\theta})$ and the balanced
dynamics 
\[
\partial_{t}A_{\theta(t)}=-\eta\sqrt{A_{\theta}A_{\theta}^{T}}\nabla C(A_{\theta})-\eta \nabla C(A_{\theta})\sqrt{A_{\theta}^{T}A_{\theta}}
\]
 when $\sigma^{2}w\ll\lambda_{min}(A_{\theta})$. But it also reveals
the whole spectrum of mixed dynamics in between, where some singular
values of $A_{\theta}$ are below the $\sigma^{2}w$ threshold and
some are above it. 

This suggests that the lazy/active transition is best understood at
a more granular level, where at each time $t$ every singular value
of $A_{\theta}$ can either be lazy or active/balanced. The mixed
regime is the best of both worlds: on one hand, since $\sqrt{A_{\theta}A_{\theta}^{T}+\sigma^{4}w^{2}I}$
is always positive definite, the network can never get stuck at a
saddle/local minimum as can happen in the balanced regime, on the other hand
there is a momentum effect where the dynamics along large singular
values is much faster than along the small ones, leading to an incremental learning behavior and a low-rank bias, which is absent in lazy learning. By choosing
the threshold $\sigma^{2}w$ adequately, one can best take advantage
of these two phenomenon.

Finally, we focus on the task of recovering a low-rank $d\times d$
matrix $A^{*}$ from noisy observations $A^{*}+E$, training on the
MSE error $\frac{1}{d^{2}}\left\Vert A_{\theta}-(A^{*}+E)\right\Vert^2_F $
in the limit as the dimension $d$, width $w$ and variance $\sigma^{2}$
scale together with scaling laws $w=d^{\gamma_{w}}$ and $\sigma^{2}=d^{\gamma_{\sigma^{2}}}$.
We describe the training dynamics for almost all reasonable scalings
$\gamma_{w},\gamma_{\sigma^{2}}$ leading to a phase diagram with
two main regimes:
\begin{itemize}
\item \textbf{Lazy ($1<\gamma_{\sigma^{2}}+\gamma_{w}$)} where all singular
values remain below the threshold $\sigma^{2}w$ throughout training,
and where the network fails to recover $A^{*}$ due to the absence
of low-rank bias.
\item \textbf{\textit{\emph{Active (}}}\textbf{\emph{$1>\gamma_{\sigma^{2}}+\gamma_{w}$}}\textbf{\textit{\emph{)}}}\textit{\emph{
where $K=\mathrm{Rank}A^{*}$ singular values pass the threshold and fit $A^{*}$
before the other singular values have time to fit the noise $E$,
leading to the recovery of $A^{*}$.}}
\end{itemize}
There are two other degenerate regimes that we avoid: the underparametrized
regime when $w<d$ (or $\gamma_{w}\ll1$) where the rank is constrained
by the network architecture rather than the training dynamics, and
the noisy regime $2\gamma_{\sigma^{2}}+\gamma_{w}+1>0$ where the
variance of the entries of $A_{\theta(0)}$ at initialization is infinite.

\begin{figure}
    \centering
    \includegraphics[width=.49\linewidth]{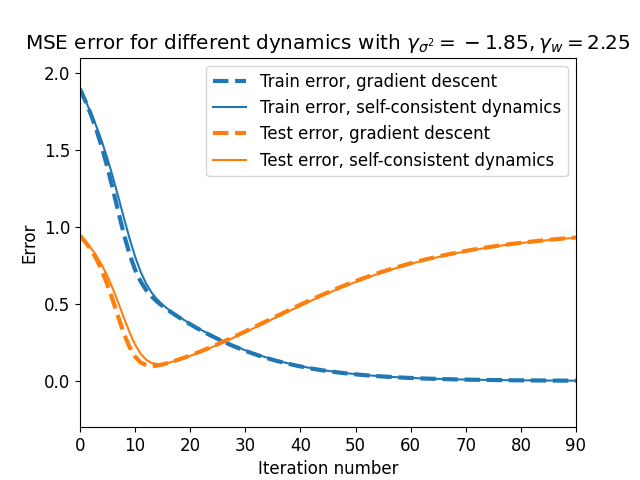}
    \includegraphics[width=.49\linewidth]{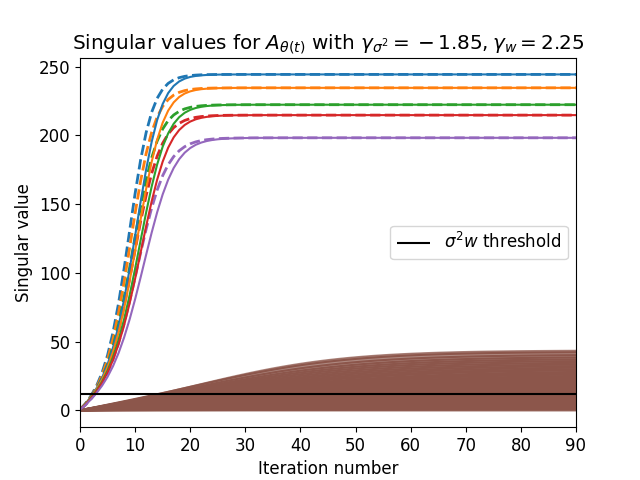}\\
    \caption{For both plots, we train either using gradient descent or the self-consistent dynamics from equation \eqref{eq:self-consistent}, with the scaling $\gamma_{\sigma^2}=-1.85,$ $\gamma_{w} = 2.25$ which lies in the active regime. (Left panel): We plot train and test error for both dynamics. We observe that the train/test error for gradient descent is very close to the train/test error for the self-consistent dynamics. (Right panel): We plot with a solid line the singular values of $A_{\theta(t)}$ when running the self-consistent dynamics, and use a dashed line for the singular values from running gradient descent. In this experiment, $\text{Rank} A^\star = 5.$ We use different colors for the $5$ largest singular values and the same color for the remaining singular values. We can see how the $5$ largest singular values `speed up' as they cross the $\sigma^2 w$ threshold, allowing them to converge earlier than the rest. The minimal test error is achieved in the short period where the large singular values have converged but not the rest.}
    \label{fig:mixed}
\end{figure}

\subsection{Previous Works}

Linear networks have been used as a testing ground, a stepping stone
on the way to understand nonlinear networks. Linear networks and their
training dynamics are in many ways much simpler than nonlinear ones,
but in spite of a long research history, our understanding remains
limited.

The setting that is best understood is that of diagonal linear networks
where the dynamics decouple along the diagonal entries leading to
an incremental learning behavior and a sparsity bias \cite{saxe_2014_exact,advani_2017_independent_diag,Saxe_2019_independent_diag,gidel_2019_independent_diag,Pesme_2023_S-to-S_diag_DLN},
some of this analysis has been extended to include effects of initialization
scale \cite{woodworth_2020_diagnet_bias} and SGD \cite{Pesme_2021_SGD_bias_diag_nets}.
While the same decoupling happens in general linear with diagonal
initializations and diagonal task, it remains an extremely strong
assumption.

Some work has been done to prove similar incremental learning dynamics
outside the diagonal case \cite{li2020towards,jacot-2021-DLN-Saddle,jin_2023_incremental_learning} where the incremental aspect can be understood as the parameters going from saddle to saddle.
For shallow linear networks, the training dynamics with MSE can be explicitly
solved \cite{Fukumizu1998_solution_shallow_DLN} but remain very complex
so that one needs to assume some form of alignment to guarantee convergence
\cite{Braun_2022_exact_DLN}. For deeper networks there exists explicit
formulas in the mean-field limit where the number of neurons grows
to infinity \cite{chizat_2022_mean_field_DLN}, these results can of course be applied to the special case of shallow nets, our paper goes further by giving self-consistent dynamics for the full matrix, revealing the lazy/active transition, and also extends the analysis to finite widths.

A very powerful tool in the analysis of a linear network is its training
invariants, and the balancedness condition which greatly simplifies
the dynamics \cite{arora_2018_DLN_convergence,arora_2018_depth_speed_Lp}.
Balanced networks exhibit a momentum effect, where the training dynamics
along a singular value $s_{i}$ have `speed' proportional to $s_{i}$
itself (or $s_{i}$ to some power), while this momentum effect seems
to be key to understand the low-rank bias of linear networks \cite{arora_2019_matrix_factorization},
it also means that one needs to guarantee that the dynamics never
approach zero, which is one the main hurdle towards proving convergence
in balanced networks. To solve this issue, recent work has focused
on initialization that slightly imbalanced \cite{xiong_2023_overparam_balanced_hurt_conv,min_2021_imbalance_shallow_DLN,tarmoun_2021_GF_matrix_fact,min_2023_imbalance_deep,xu_2023_imbalance_GD}.
This suggests that it is key to find the right balance between balancedness
and imbalancedness to obtain both fast convergence and low-rank bias.

In a concurrent work \cite{kunin_2024_richquickexactsolutions} a similar transition between lazy and active regimes is observed, and the same mixed dynamics are derived for a specific initialization. In contrast, we  prove that these dynamics are approximately true with high probability for random Gaussian initializations, which is the standard initialization scheme for neural networks.

\subsection{Setup}

We will study shallow linear networks (or matrix factorization) where
a $d_{out}\times d_{in}$ matrix $A_{\theta}$ is represented as the
product of two matrices $A_{\theta}=W_{2}W_{1}$, where the weight
matrices $W_{1}$ and $W_{2}$ are respectively $w\times d_{in}$
and $d_{out}\times w$ dimensional, for some width $w$. The parameters
$\theta$ of the network are the concatenation of the entries of both
submatrices $\theta=(W_{1},W_{2})$.

The parameters $\theta$ are learned in the following manner: they
are initialized as i.i.d. Gaussian $\mathcal{N}(0,\sigma^{2})$, and
then optimized with gradient descent to minimize a loss $\mathcal{L}(\theta)=C(A_{\theta})$.
Though most of our analysis works for general convex costs $C:\mathbb{R}^{d_{out}\times d_{in}}\to\mathbb{R}$
on matrices, we will in the second part focus on the task of recovering
a low-rank matrix $A^{*}$ from noisy observations $A^{*}+E$, by
training a linear network $A_{\theta}$ on the MSE loss
\[
\mathcal{L}(\theta)=\frac{1}{d^2}\left\Vert A_{\theta}-(A^{*}+E)\right\Vert _{F}^{2}.
\]

The width $w$ allows us to control the over parametrzation, indeed
the set of matrices that can be represented by a network of width
$w$ is the set $\mathcal{M}_{\leq w}$ of matrices of rank $w$ or
less. The overparametrized regime is when $w\geq\min\{d_{in},d_{out}\}$
because all matrices can be represented in this case.

\subsection{Lazy Dynamics}

The evolution of the weight matrices during gradient descent with
learning rate $\eta$ is given by
\begin{align*}
W_{1}(t+1) & =W_{1}(t)-\eta W_{2}^{T}(t)\nabla C(A_{\theta(t)})\\
W_{2}(t+1) & =W_{2}(t)-\eta\nabla C(A_{\theta(t)})W_{1}^{T}(t)
\end{align*}
where we view the gradient $\nabla C(A_{\theta(t)})$ of the cost
$C$ as a $d_{out}\times d_{in}$ matrix, which for the MSE cost equals
$\nabla C(A_{\theta(t)})=2d^{-2}(A_{\theta(t)}-(A^{*}+E))$.

But we care more about the evolution of the complete matrix $A_{\theta(t)}=W_{2}(t)W_{1}(t)$
induced by the evolution of $W_{1}(t),W_{2}(t)$, which can be approximated
by
\begin{equation}
    \label{eq:gd}
    A_{\theta(t+1)}=A_{\theta(t)}-\eta W_{2}(t)W_{2}^{T}(t)\nabla C(A_{\theta(t)})-\eta\nabla C(A_{\theta(t)})W_{1}^{T}(t)W_{1}(t)+O(\eta^{2}).
\end{equation}
Thus we see that if we can describe the matrices $C_{1}=W_{1}^{T}W_{1}$
and $C_{2}=W_{2}W_{2}^{T}$ throughout training, then we can describe
the evolution of $A_{\theta(t)}$. 

When $w$ is very large, we end up in the lazy regime where the parameters
move enough up to a time $t$ to change $A_{\theta(t)}$, but not
enough to change $C_{1},C_{2}$\footnote{ To be more precise the direction in parameter space that change $C_1,C_2$ are approximately orthogonal to those that change $A_\theta$, and GD/GF only moves along the later direction.}, allowing us to make the approximation $C_{i}(t)\approx C_{i}(0)$.
Furthermore at initialization these matrices concentrate as $w\to\infty$
around their expectations $\mathbb{E}\left[C_{1}\right]=\sigma^{2}wI_{d_{in}}$,
$\mathbb{E}\left[C_{2}\right]=\sigma^{2}wI_{d_{out}}$. The GD dynamics
can then be approximated by the much simpler dynamics: 
\[
A_{\theta(t+1)}=A_{\theta(t)}-2\eta\sigma^{2}w\nabla C(A_{\theta(t)}),
\]
which are equivalent to doing GD on the cost $C$ directly with a
learning rate of $2\eta\sigma^{2}w$.

One can then easily prove exponential convergence for any convex cost
$C$ following the convergence analysis of traditional linear models.
But we can see the absence of feature learning from the fact that
the covariance $C_{1}$ of the `feature map' $W_{1}$ is (approximately)
constant. More problematic in the context of low-rank matrix recovery
is the absence of low-rank bias, indeed one can easily solve the dynamics
to obtain 
\[
A_{\theta(t)}=(A^{*}+E)+(1-4d^{-2}\eta\sigma^{2}w)^{t}(A_{\theta(0)}-(A^{*}+E)),
\]
and since $\mathbb{E}A_{\theta(0)}=0$ we obtain
\[
\mathbb{E}\left[A_{\theta(t)}\right]=\left(1-(1-4d^{-2}\eta\sigma^{2}w)^{t}\right)(A^{*}+E).
\]
The expected test error $\mathbb{E}\left\Vert A_{\theta(t)}-A^{*}\right\Vert ^{2}$ is
therefore lower bounded by
\[
\left\Vert \mathbb{E}A_{\theta(t)}-A^{*}\right\Vert ^{2}=\left\Vert (1-4d^{-2}\eta\sigma^{2}w)^{t}A^{*}+(1-(1-4d^{-2}\eta\sigma^{2}w)^{t})E\right\Vert ^{2}
\]
which never approaches zero.

In linear networks, there is no advantage to being in the lazy regime,
as we simply recover a simple linear model at an additional cost of
more parameters and thus more compute. But we will see that a short
period of lazy regime at the beginning of training plays a crucial
role in making sure that the subsequent active regime starts from
an `aligned' state.

\subsection{Balanced Dynamics}

There has been much more focus on so-called balanced linear networks,
which are networks that satisfy the balanced condition $W_{1}W_{1}^{T}=W_{2}^{T}W_{2}$. If the network is balanced at initialization, it remains so throughout training, because,
the difference $W_{1}W_{1}^{T}-W_{2}^{T}W_{2}$ is an invariant of GF (and an approximate invariant of GD with small enough learning rate).

First observe that the balanced condition implies the following shared eigendecomposition $W_{1}W_{1}^{T}=W_{2}^{T}W_{2}=USU^{T}$. This
implies the following shared SVD decompositions $W_{1}=U\sqrt{S}U_{in}^{T}$,
$W_{2}=U_{out}\sqrt{S}U^{T}$ and $A_{\theta}=U_{out}SU_{in}^{T}$.
Furthermore, we have $C_{1}=U_{in}SU_{in}^{T}=\sqrt{A_{\theta}^{T}A_{\theta}}$
and $C_{2}=U_{out}SU_{out}^{T}=\sqrt{A_{\theta}A_{\theta}^{T}}$,
which leads to self-consistent dynamics for $A_{\theta(t)}$:
\[
A_{\theta(t+1)}=A_{\theta(t)}-\eta\sqrt{A_{\theta(t)}A_{\theta(t)}^{T}}\nabla C(A_{\theta(t)})-\eta\nabla C(A_{\theta(t)})\sqrt{A_{\theta(t)}^{T}A_{\theta(t)}}+O(\eta^{2}).
\]

Now these dynamics are quite complex in general, and it remains difficult
to prove convergence. Indeed one can easily find initializations
$A_{\theta(0)}$ that will not converge,
for example if $A_{\theta(0)}=0$ then GD will remain stuck there.
A lot of work has been dedicated to finding conditions that guarantee
the convergence of the above dynamics \cite{arora_2018_DLN_convergence,Braun_2022_exact_DLN}, but these assumptions are often quite strong.

The simplest initialization that guarantees convergence (and the one that will be most relevant
to our analysis) is the positively aligned initialization. If at initialization $A_{\theta(0)}$ and $A^{*}+E$ are `aligned',
i.e. shares the same singular vectors $A_{\theta(0)}=U_{out}SU_{in}^{T}$
and $A^{*}+E=U_{out}S^{*}U_{in}^{T}$, then they will remain aligned
throughout training $A_{\theta(t)}=U_{out}S(t)U_{in}^{T}$ and the
dynamics decouple along each singular value
\[
s_{i}(t+1)=s_{i}(t)+2\eta\left|s_{i}(t)\right|\left(s_{i}^{*}-s_{i}(t)\right)+O(\eta^{2}).
\]
Since we always have $s_{i}^{*}\geq0$, then for small enough
learning rates $\eta$, we see that if $s_{i}(0)\in(0,s_{i}^{*}]$
it will grow monotonically and converge to $s_{i}^{*}$; if $s_{i}(0)>s_{i}^{*}$
it will decrease monotonically to $s_{i}^{*}$, and if $s_{i}(0)\leq0$
it will increase and converge to $0$. Thus one can guarantee convergence
if we further assume positive alignment $s_{i}(0)>0$.

The advantage is that there is a momentum effect in the form
of the prefactor $\left|s_{i}(t)\right|$, which implies that the
dynamics along large singular values are faster than along small ones.
As a result, if all singular values are initialized with the same
small value, then they will at first grow very slowly until they reach
a critical size where the momentum effect will make them converge
very fast. The singular values aligned with the top singular values
of $A^{*}+E$ will reach this threshold much faster, and they will
therefore converge to approximately their final value $s_{i}=s_{i}^{*}$
at a time when the other singular values are still basically zero. If we stop
training at this time then the linear network will have essentially
learned only the top $K$ singular values of $A^{*}+E$, which is
a good approximation for $A^{*}$, leading to a small test error (see \cite{gidel_2019_independent_diag} for details).

But this analysis relies on the very strong assumption of positive
alignment at initialization. If we do not assume a positive alignment
and assume that the $s_{i}$ are random (i.i.d. w.r.t. a symmetric
distribution), then each $s_{i}$ has probability $\nicefrac{1}{2}$
of starting with a negative alignment and getting stuck at zero,
which means that with high probability training will fail to recover
$A^{*}$ and will recover only a random subset of the singular values
of $A^{*}$. The presence of these attractive saddles shows the complexity
of the balanced dynamics.

A limitation of this approach is that it requires a quadratic cost
and a very specific initialization, and in the case of positive alignment,
an initialization that requires knowledge of the (SVD of the) true
function $A^{*}$. Nevertheless, the positively aligned and balanced
dynamics seem to capture some qualitative phenomenon that has been
observed empirically outside of this restricted setting. This is the
phenomenon of incremental learning, where if the singular values are
initialized as very small, they first grow very slowly, but the multiplicative momentum will lead to come
up one by one in a very abrupt manner, and this leads to a low rank
bias where the network first only fits the largest singular value,
then two largest, and so on. More generally, this can be interpreted as the
network performing a greedy low-rank algorithm \cite{li2020towards}.

Our analysis will confirm the fact that positive alignment happens
naturally as a result of a short period of lazy training, allowing
us to prove similar decoupling and incremental learning for a general
random initialization.

\begin{rem}
We can define the time dependent map $\Theta(G;t)=C_{2}(t)G+GC_{1}(t)$,
so that the GD dynamics can be rewritten as $A_{\theta(t+1)}=A_{\theta(t)}-\eta\Theta(\nabla C(A_{\theta(t)}),t)+O(\eta^{2})$.
The map $\Theta$ is none other than the NTK for shallow linear networks,
but it has also been called the preconditioning matrix in previous
work \cite{arora_2018_depth_speed_Lp}. The lazy regime is then characterized
by the NTK $\Theta$ being approximately equal to the time-independent
NTK $\Theta^{\text{lazy}}(G)=2\sigma^{2}wG$, whereas the balanced
regime is characterized by the time-dependent $\Theta^{\text{bal}}(G;t)=\sqrt{A_{\theta(t)}A_{\theta(t)}^{T}}G+G\sqrt{A_{\theta(t)}^{T}A_{\theta(t)}}$,
with the distinction that the time dependence is only through $A_{\theta(t)}$.
\end{rem}

\section{Mixed Lazy/Balanced Dynamics}

Both lazy and balanced dynamics have the surprising but very useful
property that the evolution of the network matrix $A_{\theta}$ is approximately  self-consistent: the evolution of $A_\theta$ can
be expressed in terms of itself. The lazy approximation
becomes correct for a sufficiently large initialization, while the balanced
one is correct for a balanced initialization. However, for most initializations, neither of these approximations are correct.

We fill this gap by providing a self-consistent evolution of $A_{\theta}$
that applies for any initialization scale:
\[
\partial_{t}A_{\theta(t+1)}\approx-\eta\sqrt{A_{\theta(t)}A_{\theta(t)}^{T}+\sigma^{4}w^{2}I}\nabla C(A_{t})-\eta\nabla C(A_{t})\sqrt{A_{\theta(t)}^{T}A_{\theta(t)}+\sigma^{4}w^{2}I}.
\]
This approximation is formalized in the following theorem, denoting
$\hat{C}_{1}(t)=\sqrt{A_{\theta(t)}^T A_{\theta(t)}+\sigma^{4}w^{2}I}$
and $\hat{C}_{2}(t)=\sqrt{A_{\theta(t)}A_{\theta(t)}^T +\sigma^{4}w^{2}I}$
\begin{thm}
\label{thm:mixed_dynamics}For a linear net $A_{\theta}=W_{2}W_{1}$
with width $w$, initialized with i.i.d. $\mathcal{N}(0,\sigma^{2})$
weights and trained with Gradient Flow,
we have with high probability that for all time $t$,
\begin{align*}
\left\Vert C_{1}(t)-\hat{C}_{1}(t)\right\Vert _{op},\left\Vert C_{2}(t)-\hat{C}_{2}(t)\right\Vert _{op} & \leq\min\left\{ O(\sigma^{2}w),O\left(\sqrt{\frac{d}{w}}\left\Vert C_{1}(t)\right\Vert _{op}\right)\right\} .
\end{align*}
\end{thm}

\begin{proof}
(sketch) The quantity $W_{1}W_{1}^{T}-W_{2}^{T}W_{2}$ is invariant
under GF (and approximately so under GD) and it is approximately equal
to $\sigma^{2}w(P_{1}-P_{2})$ for two orthogonal projections $P_{1},P_{2}$
(at initialization and for all subsequent times because of the invariance).
We therefore have
\[
W_{1}^{T}(W_{1}W_{1}^{T}-W_{2}^{T}W_{2})^{2}W_{1}\approx\sigma^{4}w^{2}W_{1}^{T}(P_{1}+P_{2})W_{1}\approx\sigma^{4}w^{2}C_{1}.
\]
Thus the pairs $C_{1},C_{2}$ approximately satisfy the following equations:
\begin{align*}
0 & \approx C_{1}^{3}-A_{\theta}^{T}A_{\theta}C_{1}-C_{1}A_{\theta}^{T}A_{\theta}-\sigma^{4}w^{2}C_{1}+A_{\theta}^{T}C_{2}A_{\theta}\\
0 & \approx C_{2}^{3}-A_{\theta}A_{\theta}^{T}C_{2}-C_{2}A_{\theta}A_{\theta}^{T}-\sigma^{4}w^{2}C_{2}+A_{\theta}C_{1}A_{\theta}^{T}.
\end{align*}
The pair $\hat{C}_{1},\hat{C}_{2}$ is a solution of the above,
and one can show that $C_{1},C_{2}$ must approach them and not any
of the other solutions.
\end{proof}
The takeaway from theorem \ref{thm:mixed_dynamics} is the following. 
\begin{enumerate}
    \item In the lazy regime where $\|C_1(t)\|_{op} + \|C_2(t)\|_{op}\leq O(\sigma^2w)$, then $\|C_1(t) - \hat{C}_1(t)\|_{op} \leq \sqrt{d/w} \|C_1(t)\|_{op} << \|C_1(t)\|_{op}$.
    \item In the active regime where $\|C_1(t)\|_{op} / \sigma^2w > d^{\varepsilon}>>1$, then $\|C_1(t) -\hat{C}_1(t)\|_{op}\leq O(\sigma^2w) <<d^{-\varepsilon}\|C_1(t)\|_{op}$. 
\end{enumerate}
It is true that the error does not vanish. However, for our purpose it suffices to show that $\|\hat{C}_1-C_1\|_{op}$ is infinitely smaller than $C_1$ for all times, regardless of the magnitude of $\|C_1(t)\|_{op}$.

We see how both the lazy and balanced dynamics appear as special cases
depending on how large the variance at initialization $\sigma^{2}$
is in comparison to the singular values of the matrix $A_{\theta(t)}$:
\begin{itemize}
\item \textbf{Lazy: }When $\sigma^{2}w\gg s_{max}(A_{\theta(t)})$, then
$\hat{C}_{1}\approx\sigma^{2}wI_{d_{out}}$ and $\hat{C}_{2}\approx\sigma^{2}wI_{d_{in}}$,
recovering the lazy dynamics.
\item \textbf{Balanced:} When $\sigma^{2}w\ll s_{min}(A_{\theta(t)})$,
then $\hat{C}_{1}\approx\sqrt{A_{\theta(t)}^{T}A_{\theta(t)}}$ and
$\hat{C}_{2}\approx\sqrt{A_{\theta(t)}A_{\theta(t)}^{T}}$, recovering
the balanced dynamics.
\end{itemize}
But clearly there can be times when neither conditions are satisfied,
when some singular values of $A_{\theta(t)}$ are larger than the
threshold $\sigma^{2}w$ while others are smaller, in such cases we
are in a mixed regime, where the network is lazy along the small singular
values of $A_{\theta(t)}$ ($s_{i}\ll\sigma^{2}w$) and active/balanced
along the large ones ($s_{i}\gg\sigma^{2}w$). 

At initialization, the singular values are of size $\sigma^{2}\sqrt{wd}$.
This implies that with overparametrization ($w\gg d$),
all singular values start in the lazy regime and follow the simple
lazy dynamics, which may (or may not) lead to some singular growing
and crossing the $\sigma^{2}w$ threshold, at which point they will
switch to balanced dynamics (after a short transition period when the singular
value is around the threshold $s_{i}\approx\sigma^{2}w$). Once a singular
value is far past the threshold $s_{i}\gg\sigma^{2}w$, training along
this singular value will be much faster than along the lazy singular
values (this speed up can be seen in Figure \ref{fig:mixed}). This allows the newly active singular values to converge while the lazy singular values remain almost constant. Once the active
singular values have converged, the slow training of the remaining
lazy singular values continues until some of these singular values reaches
the threshold, or until GD converges.

This type of behavior is illustrated by the following formula, which
describes the derivative in time of the $i$-th singular value $s_{i,t}$
of $A_{t}$, with singular vectors $u_{i,t},v_{i,t}$:
\[
s_{i,t+1}-s_{i,t}\approx\eta_{t}u_{i,t}^{T}\partial_{t}A_{\theta(t)}v_{i,t}\approx-2\eta_{t}\sqrt{s_{i,t}^{2}+\sigma^{4}w^{2}}u_{i,t}^{T}\nabla C(A_{\theta(t)})v_{i,t},
\]
where the prefactor $2\eta_{t}\sqrt{s_{i,t}^{2}+\sigma^{4}w^{2}}$
describes the effective learning rate along the $i$-th singular value,
which depends on the $i$-th singular value $s_{i,t}$ itself.

This suggests that it is more natural to distinguish between the lazy
and active regime at a much more granular level: at every time $t$
a singular value can be either active or lazy (or very close to the
transition but this typically only happens for a very short time).
In contrast, the traditional definition of the lazy regime was defined
for a whole network and over the whole training time. To avoid confusion,
we call this the pure lazy regime, where all singular values remain
lazy throughout training. This begs the question of whether a pure
balanced regime also exists, but all singular values will always be
lazy for at least a short time period (assuming $w>d$), and as we will see this short
lazy period plays a crucial role in aligning the network so that the
subsequent balanced regime can learn successfully. A pure balanced
regime can only be obtained in the underparametrized regime, or by taking  a balanced initialization instead of
the traditional i.i.d. random initialization.

While this challenges the traditional lazy/active dichotomy, it also
reinforces it, as it shows that there is no fundamentally different
third regime, only lazy, active, and some mix of the two. Theorem
\ref{thm:mixed_dynamics} thus allows us to revisit previous descriptions
of lazy and balanced dynamics and `glue them together' to extend them
to the general case. This simple strategy will allow to almost fully
`fill in the phase diagram', i.e. describe the dynamics, convergence
and generalization properties of DLNs for almost all reasonable initialization
scales $\sigma^{2}$ and widths $w$.
\begin{rem}
The transition of a singular value $s_{i}$ from lazy to active can
be understood as a form of alignment happening in the hidden layer:
the two vectors $W_{1}v_{i}$ and $W_{2}^{T}u_{i}$ for $u_{i},v_{i}$
the left and right singular vectors of $s_{i}$ are orthogonal in
the lazy regime and become perpendicular in the balanced regime. Indeed
the normalized scalar product of these two vectors satisfies
\[
\frac{u_{i}^{T}W_{2}W_{1}v_{i}}{\left\Vert W_{2}^{T}u_{i}\right\Vert \left\Vert W_{1}v_{i}\right\Vert }=\frac{s_{i}}{\sqrt{u_{i}^{T}C_{2}u_{i}}\sqrt{u_{i}^{T}C_{1}u_{i}}}\approx\frac{s_{i}}{\sqrt{s_{i}^{2}+\sigma^{4}w^2}}
\]
which is close to zero for lazy singular values $s_{i}\ll\sigma^{2}w$
and close to one for active ones $s_{i}\gg\sigma^{2}w$.
\end{rem}

\begin{figure}
    \centering
    \includegraphics[width=.49\linewidth]{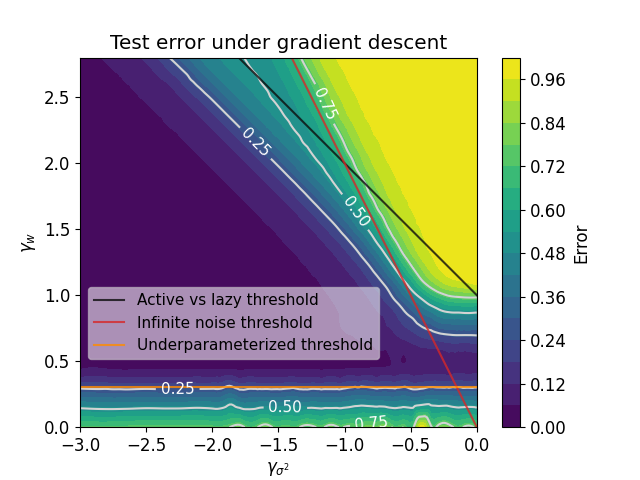}
    \includegraphics[width=.49\linewidth]{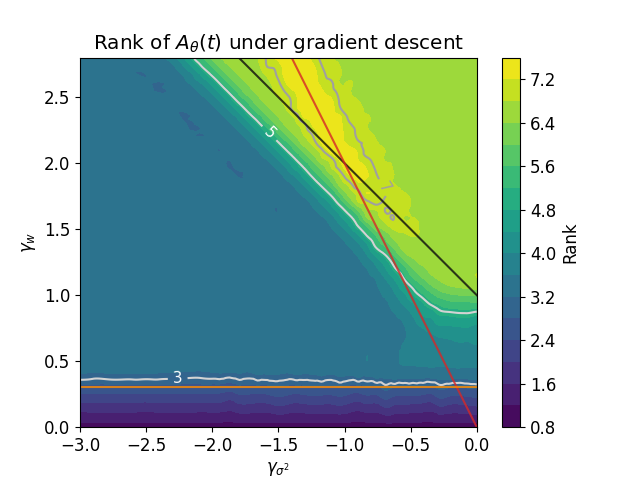}\\
    \includegraphics[width=.49\linewidth]{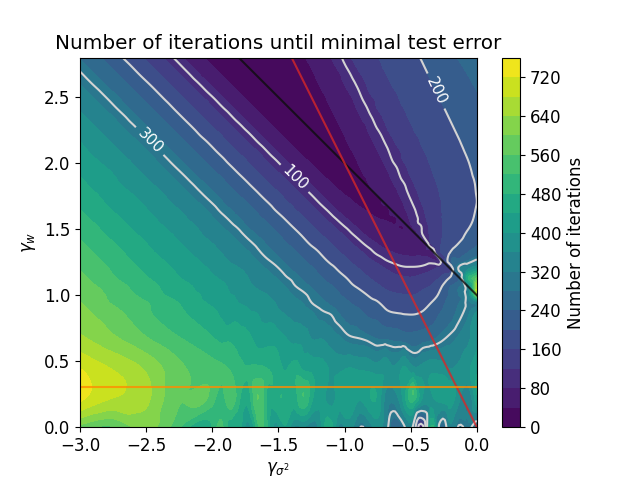}
    \includegraphics[width=.49\linewidth]{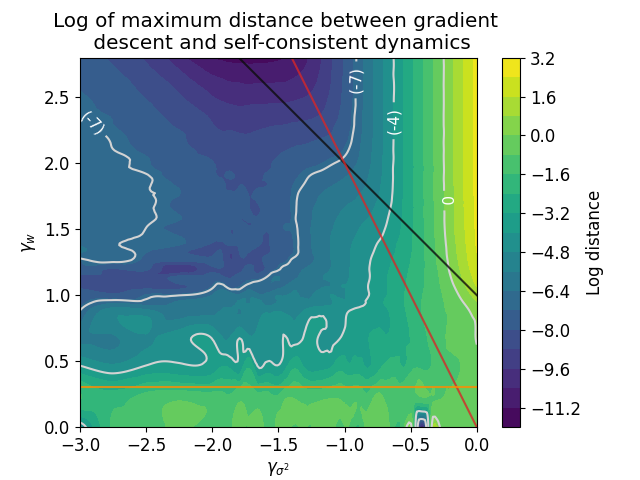}\\
    \caption{As a function of $\gamma_{\sigma^2}, \gamma_w,$ we run GD and plot different quantities. Our theoretical results only apply to the top left region for $\gamma_w>1$ and below the red line, although these plots suggest that some results may extend to smaller $\gamma_w$s. (Top left panel): We plot the smallest test error $\frac{1}{d^2}\Vert A_{\theta(t)}-A^{*}\Vert_F ^{2}$ in the whole run. The active region (below the black line) has a small error while the lazy region does not. (Top right panel): We plot the stable rank of $A_{\theta(t)}$ (defined as $\Vert A_{\theta(t)}\Vert_F ^{2} / \Vert A_{\theta(t)} \Vert_\text{op} ^{2}$) at the time of minimal test error. In this experiment, we took $\text{Rank} A^{*} = 5.$ We see that the active region has approximately the correct rank while the lazy region overestimates it. (Bottom left panel): We plot the number of iterations until minimal test error, illustrating the trade-off between test error and training time. (Bottom right panel): We compute $\ln \left(\frac{1}{d^2}\Vert A_{\theta(t)}-\hat{A}_{\theta(t)}\Vert_F ^{2}\right)$ where $A_{\theta(t)}$ comes from GD and $\hat{A}_{\theta(t)}$ from the self-consistent dynamics. We observe that this distance is not only small for the region where our theoretical results apply but also almost everywhere outside this region.}
    \label{fig:heatmaps}
\end{figure}

\subsection{Phase Diagram for MSE}

To illustrate the power of Theorem \ref{thm:mixed_dynamics} we provide
a phase diagram of the behavior of large shallow networks on a MSE
task, for almost all (reasonable) choices of width $w$ and variance
$\sigma^{2}$ scalings.

We want to recover a rank $K$ and $d\times d$-dimensional matrix
$A^{*}$ with $s_{i}(A^{*})=da_{i}$ for some $a_{1}\geq a_{2}\geq\dots\geq a_{K}$
independent of the dimension $d$. We however only observe a noisy
version $A^{*}+E$ for some $E$ such that $\left\Vert E\right\Vert _{op}\leq c_{0}d^{\delta}$.
One could imagine $E$ to have iid random Gaussian entries $\mathcal{N}(0,1)$
in which case $\left\Vert E\right\Vert _{op}\leq c_{0}\sqrt{d}$ with
high probability.

As the dimension $d$ grows, the size of the network needs to scale
too, as well as the initialization variance, but it is unclear what
is the optimal way to choose $w$ and $\sigma^{2}$. We will therefore
consider general scalings $w=d^{\gamma_{w}}$ and $\sigma^{2}=d^{\gamma_{\sigma^{2}}}$.
We will now describe the $(\gamma_{w},\gamma_{\sigma^{2}})$-phase
diagram which features 4 regimes: underparametrized, infinite-noise,
lazy and mixed/active.

We can identify a region of `reasonable' pairs $(\gamma_{\sigma^{2}}, \gamma_{w})$
by ruling out degenerate behavior. First, the width $w$ needs to
be larger than the dimension $d$, since a network of width $w$ can
only represent matrices of rank $w$ or less, this means that we need
$\gamma_{w}\geq 1$. Another constraint comes from the variance
of $A_{\theta}$ at initialization: the entries $A_{\theta(0),ij}$
at initialization have variance $\sigma^{4}w$. We
want this variance to go to zero as $d$ grows which implies that
we need $2\gamma_{\sigma^{2}}+\gamma_{w} < 0$.

Now within this reasonable region we observe two regimes, the pure
lazy regime for $1<\gamma_{\sigma^{2}}+\gamma_{w}$ where the network
simply fits $A^{*}+E$ thus failing to learn $A^{*}$ and the mixed
regime for $1>\gamma_{\sigma^{2}}+\gamma_{w}$ where the dynamics
are lazy for a short amount of time until $K$ singular values grow
large enough to switch to the balanced dynamics and fit the true matrix
$A^{*}$. 
\begin{thm}
\label{thm: main theorem}
For pairs $\gamma_{w},\gamma_{\sigma^{2}}$ such that $\gamma_{w}>1$
and $2\gamma_{\sigma^{2}}+\gamma_{w}<0$, we have two regimes:
\begin{itemize}
\item \textbf{Lazy ($1<\gamma_{\sigma^{2}}+\gamma_{w}$):} with a learning
rate $\eta\ll\frac{d^2}{\sigma^{2}w}$ we have that for all time $t$,
$\frac{1}{d^{2}}\left\Vert A_{\theta(t)}-A^{*}\right\Vert _{F}^{2}\geq c$.
\item \textbf{Active ($1>\gamma_{\sigma^{2}}+\gamma_{w}$):} with a learning
rate $\eta\ll\frac{d^2}{s_{1}(A^{*})}\sim d$, and at time 
$$
t = \frac{1}{\eta}\left(\frac{\Delta}{a_K} + \frac{2\max(1,2\Delta)}{c(a_1,\ldots,a_K)} +  \frac{\max(1,2\Delta)}{2a_K} \right)d\log d+\eta^{-1}O(d\log\log d),
$$
for $\Delta=1-\gamma_{\sigma^2} - \gamma_w>0$, we have that 
$$\frac{1}{d^{2}}\left\Vert A_{\theta(t)}-A^{*}\right\Vert _{F}^{2}\leq O(\sigma^{4}w+\frac{\sigma^4w^2\log ^2 d}{d^2} + d^{-\frac{1}{2}} + \frac{\sigma^2w}{d} + \eta^2 \frac{\log^2 d}{d^2}),$$
for $
c(a_1,\ldots, a_K) = \frac{\min_{k,j:a_k\neq a_j}\abs{a_k-a_j} a_K^2}{\max_{k,j: a_k\neq a_j} \abs{a_k^2 - a_j^2}}.
$
\end{itemize}
\end{thm}

Note that all the terms inside the final $O(\dots)$ term vanish: $\sigma^4 w \to 0$ because $\gamma_{\sigma^2} + \gamma_w<0$, $\frac{\sigma^4w^2\log ^2 d}{d^2} + \frac{\sigma^2w}{d} \to 0$ since $1>\gamma_{\sigma^2} + \gamma_w$, and $\eta^2 \frac{\log^2 d}{d^2}\to 0$ since we assumed $\eta \ll d$.

This shows that the lazy regime only appears for very large widths $\gamma_w>2$ (or at least the lazy regime with finite variance at initialization). Indeed the choice $\gamma_w=2,\gamma_{\sigma^2}=-1$ is at the boundary of the lazy regime with the smallest $\gamma_w$. This could explain why it is rare to observe the lazy regime in practice.

Our theoretical results applies to the overparametrized regime $w \gg d$, but actually we only want to fit $A^*$ which has a much smaller rank $r$, and so we might only need $w \gg r$. Figure \ref{fig:heatmaps}, top left panel, confirms this, since we see a good generalization even for small widths $w < d$, and in particular when $w\approx \mathrm{Rank}A^*$. But to leverage this underparametrized regime, one would need to know the rank of the true matrix $A^*$ in advance, which is typically not the case in practice. Nevertheless, the interesting behavior we observe in the (mildly) underparametrized regime warrants further analysis, and the fact that our self-consistent dynamics remain a good approximation in this regime (Figure \ref{fig:heatmaps}, bottom right panel), suggests that the analysis we present here could be extended to this regime too.

Finally, we observe a trade-off between generalization error and training time: on one hand the test error has terms that scale negatively with $1 - \gamma_{\sigma^2} - \gamma_w$, which is the distance to the lazy/active transition, on the other hand, the time it takes to reach the minimal loss point scales positively with the same term. This can be seen from Figure \ref{fig:heatmaps}, bottom left panel, which plots the number of steps required to reach minimal test error, which increases as one goes further into the active regime.

\begin{rem}
In general when trying to fit a matrix $B$ (instead of the special case $B=A^* + E$), the transition between lazy and mixed regime is when $\sigma^2 w \approx \| B \|_{op}$. Thus the exact location of the transition is task-dependent, so that the same variance $\sigma^2$ and width $w$ can lead to NTK or mixed regimes depending on the task. For example, let us assume that $A^*$ is full-rank instead of finite rank, then we expect $\| A^* \|_{op} \sim \sqrt{d}$ instead of $\| A^* \|_{op} \sim d$, thus the transition would be at $\frac 1 2 = \gamma_{\sigma^2} + \gamma_w$ instead of $1 = \gamma_{\sigma^2} + \gamma_w$. This suggests that linear networks are able to adapt themselves to the task: leveraging active dynamics when the true data is low-rank to get better generalization, or remaining in the lazy dynamics in the absence of low-rank structure, to take advantage of the faster convergence. Note also that in the absence of sparsity, the lazy regime can be attained with a smaller width ($\gamma_w > 1$ instead of $\gamma_w > 2$), since the choice $\gamma_w=1,\gamma_{\sigma^2 }=-\frac 1 2$ is already on the boundary of the lazy regime.
\end{rem}

\section{Conclusion}
We prove a surprisingly simple self-consistent dynamic for the evolution of the matrix represented by a shallow linear network under gradient descent. This description not only unifies the already known lazy and balanced dynamics, but reveals the existence of a spectrum of mixed dynamics where some of the singular values are lazy while others are balanced.

Thanks to this description we are able to give an almost complete phase diagram of training dynamics as a function of the scaling of the width and variance at initialization w.r.t. the dimension.

A natural question that comes out of these results is whether nonlinear network also feature similar mixed regimes, and whether they could be the key to understand the convergence of general DNNs.

\bibliographystyle{plain}
\bibliography{bibliography}

\newpage 
\appendix
The appendix is structured as follows. 
\begin{itemize}
    \item In section \ref{sec: prelim}, we introduce the notation, and establish several results about how perturbing a matrix would impact its singular vectors. \item In section \ref{sec: GF for A_t}, we study the gradient flow dynamics of $A_t$ in the active regime and prove that $A_t$ will be approximately aligned with $A^*$ throughout the Saddle-to-Saddle regime. 
    \item In section \ref{sec: Proof of main thm 1}, we prove theorem \ref{thm:mixed_dynamics} for general cost. 
    \item In section \ref{sec: GF for A_theta in Lazy}, we study the gradient flow dynamics for $A_{\theta(t)}$ in the lazy regime. 
    \item In section \ref{sec: GF for A_theta in Active}, we show that $A_{\theta(t)}$ is also approximately aligned with $A^*$ throughout the Saddle-to-Saddle regime, using results in section \ref{sec: GF for A_t} and section \ref{sec: Proof of main thm 1}. In subsection \ref{subsection: dynamics summary}, we summarize the approximate dynamics of $A_{\theta(t)}$ throughout training. In section \ref{subsec: error analysis}, we bound the final test error.
    \item In section \ref{sec: GD and Proof of main thm 2}, we bound the error from gradient descent and prove theorem \ref{thm: main theorem}. 
    \item In section \ref{sec:exp_details}, we describe the experimental setup.
\end{itemize}

\section{Preliminaries}
\label{sec: prelim}
\subsection{Convention and Notation}

\textbf{Constants}. $d$ is the dimension of the input and the output layer, $K $ is the rank of matrix $A^*$, and $w$ is the dimension of hidden layer. $c$ and $C$ will usually denote constants that are independent of $d$, and depending on the context, the value of $c$ and $C$ might be different. If $x$ is a scalar that depends on $d$ and $y$ is a scalar, then $x = O(y)$ means there exists a constant $c$ independent of $d$, such that for $d$ sufficiently large we have $x \leq cy$. If $A$ is a matrix, then $A = O(y)$ means there exists a constant $c$ independent of $d$ such that for $d$ sufficiently large, $\|A\|_{op}\leq cy$. $O(y)$ can be either a matrix or a scalar, and its meaning will be always clear from context. \\
\\
\textbf{Matrix}. We use $\|\cdot\|_{op}, \|\cdot\|_{F}, \|\cdot\|$ to denote the operator norm of a matrix, the Frobenius norm of a matrix, and the $L^2$ norm for vectors. For every matrix $A$, we define $\max A$ as the $L^\infty$ norm, $ \max_{i,j}\abs{A_{ij}}$. We use $I$ to denote identity matrix, the dimension of which is determined by context. We shall assume that the signal singular values of $A^*$ are $s_1^*,\ldots, s_K^*$. $\forall i =1,2,\ldots, K$, and $s_i^* = a_id$ where $a_1\geq a_2\geq \ldots \geq a_K$ are constants independent of $d$. By selecting proper basis in the input and output space, we assume that $A^* = S^*$, where $S^*$ is the diagonal matrix consisting of singular values of $A^*$, ordered from largest to smallest. The $p,q$-th element of a matrix $A$ is denoted $A_{pq}$. We reserve the notation $A(i,j)$ for the $i,j$-th block matrix of $A$, which we shall define below.\\

\textbf{Submatrix}.  Assume that $n_0 = 0$, and $a_1 = \ldots = a_{n_1} > a_{n_1 +1} = \ldots a_{n_2} > \ldots = a_{n_m} = a_K$, and let $n_{m+1} = d$. For a matrix $U$, we define the $k,j$-th sub-block $U(k,j)$ of a matrix $U$ as $U_{n_{k-1}+1:n_{k+1}, n_{j-1}+1:n_{j}}$,  with both sides included. Notice that $U^T(k,j) = U(j,k)^T$. In this notation, we can write the singular value decomposition of a matrix $A$ as $$A(i,j) = \sum_{k: \text{signal}} U(i,k)S(k,k)V(j,k) + U(i,m+1)S(m+1,m+1)V^T(m+1,j).$$
We call an index $k$ (of sub-block) "signal", if $k\leq m$. Index $m+1$ is called "noise". Let $S(k,k)^*$ be block matrix $A^*(n_k:n_{k+1},n_k:n_{k+1})$. Then $A^* = \mathrm{diag}(S(1,1)^*,\ldots,S(m,m)^*)$ and $S(k,k)^* = s_{n_k} I$ is the $k$-th sub-block of $A^*$. Each matrix has only finitely many sub-blocks.    \\
\\
\textbf{Indexing Conventions}. Entries of matrices will usually be indexed by $p,q, r$ and sub-blocks of matrices will usually be indexed by $i,j,k,\ell$. Usually $k$ ranges from 1 to $m$, and $j$ usually ranges from 1 to $m+1$.  \\
\\
\textbf{Element-wise Product}.We use $\odot$ to represent element-wise product of two matrix of the same shape. 
\\
\\
\textbf{Important Assumptions}. Throughout the paper, we shall always assume that 
\begin{enumerate}
    \item $\gamma_w > 1$ (i.e., $w >> d$).
    \item $2\gamma_{\sigma^2} + \gamma_w < 0$. (i.e., $\sigma^4w <<1$). 
\end{enumerate}
\subsection{Matrix Inequalities}
In the proof of main theorems, we will work extensively with inequalities of matrix norms and inequalities that involves element-wise product. The element-wise product appears naturally in the derivative of singular vectors of a matrix.
\begin{lemma}
    Assume that $A$, $B$ and $R$ are square matrices. Let $R_{max} = \max_{i,j}\abs{R_{ij}}$. Then 
    \[
    tr[A(R\odot B)] = tr[BR^T\odot A],
    \]
    and
    \[
    \abs{tr[A(R\odot B)]}\leq R_{\max} \sqrt{tr(A^TA)} \sqrt{tr(B^TB)}
    \]
    In particular, if $\forall p,q, R_{pq}\geq R_{\min} > 0$, then 
    \[
    tr[A^T(R\odot A)] \geq R_\min tr(A^TA)
    \]
\end{lemma}
\begin{proof}
    All are simple computations. 
    \begin{align*}
        tr[AR\odot B] =& \sum_{p,q}A_{pq}R_{qp}B_{qp}
    \end{align*}
    \begin{align*}
        tr[BR^T\odot A] =& \sum_{p,q} B_{pq} R_{pq}A_{qp}
    \end{align*}
    The two equations above prove the first claim. 
    \begin{align*}
        \abs{tr[A(R\odot B)]} \leq&  \sqrt{tr(A^TA)} \sqrt{tr((R\odot B)^TR\odot B)}\\
        \leq & R_{\max} \sqrt{tr(A^TA)} \sqrt{tr(B^TB)}
    \end{align*}
    This completes the second claim.
    \begin{align*}
        tr[A^T(R\odot A)] &= \sum_{i,j} A_{ji}R_{ji}A_{ji}\\
        &\geq R_\min tr[A^TA]
    \end{align*}
    This completes the third claim.
\end{proof}

\begin{lemma}
    $\sigma_{min}(A)\|B\|_F\leq \|AB\|_F\leq \sigma_{max}(A)\|B\|_F$.
\end{lemma}
\begin{proof}
    This is lemma B.3 of \cite{zou2020global}.
\end{proof}

\subsection{Perturbation of Singular Values and Singular Vectors}

We will often use the following variant of the Davis-Kahan $\sin\theta$ theorem. 
\begin{theorem}
    [DK-$\sin\theta$ Theorem]. Let $\Sigma, \hat{\Sigma}\in \mathbb{R}^{p\times p}$ be symmetric, with eigenvalues $\lambda_1\geq \ldots \geq \lambda_p$ and $\hat{\lambda}_1\geq\ldots\geq \hat{\lambda}_p$. Fix $1\leq r\leq s\leq p$, let $d = r-s+1$ and let $V = (v_r,\ldots, v_s) $ and $\hat{V} = (\hat{v}_r,\ldots,\hat{v}_s)$ have orthonormal columns satisfying $\Sigma v_j = \lambda_jv_j$ and $\Sigma \hat{v}_j = \hat{\lambda}_j v_j$. Let $\sigma_1,\ldots, \sigma_d$ be the singular values of $\hat{V}^TV$. Let $\Theta(V,\hat{V})$ be the diagonal matrix with $\cos \Theta(V,\hat{V})_{jj} = \sigma_j$ and $\sin\Theta(V,\hat{V})$ be defined entry-wise. Then 
    \[
    \|\sin\Theta(V,\hat{V})\|_F\leq \frac{2\min(\sqrt{d} \|\hat{\Sigma} - \Sigma\|_{op}, \|\hat{\Sigma} - \Sigma\|_F)}{\min(\lambda_{r-1} - \lambda_r, \lambda_{s}- \lambda_{s+1})}.
    \]
    
\end{theorem}
\begin{proof}
    This is theorem 2 in \cite{yu2015useful}.
\end{proof}
The implication of the theorem is that if two matrices are sufficiently close, then their singular vectors are also close to each other. In the case where $r = s$ and $\lambda_r$ is of multiplicity 1, the theorem reduces to saying that the sine value of the angle between $v_r$ and $\hat{v}_r$ is very small.\\

The term $\|\sin\Theta(V,\hat{V})\|_F$ is complicated to take derivative. In this paper we will use the following characterization of alignment, which is easier to take derivatives. 
\begin{lemma}
    Let $\hat{\Sigma}$ be a $d\times d$ diagonal matrix, and let $s_1,\ldots, s_K,\ldots, s_d$ be its diagonal entries. Assume that $s_1 = \ldots = s_{n_1} > s_{n_1 + 1}= \ldots 
= s_{n_2}\geq \ldots s_{n_m} = s_K > s_{K+1} = \ldots = s_d$. Then $\hat{\Sigma}$ has $m+1$ blocks in total. Assume that $\|X-\hat{\Sigma}\|_{op} = d^{\alpha}$. Let $X = USV^T$, and define 
$$x = 4K - \sum_{k: \text{signal}} \mathrm{tr}(U^T(k,k)U(k,k) + V^T(k,k)V(k,k) + 2U(k,k)V(k,k)^T).$$ Then 
    \[
    x\leq K\frac{\|\hat{\Sigma} - X\|_{op}}{s_K} + \frac{ms_1}{s_K}\left( \frac{  2\min(
    \sqrt{K}(\|X\|_{op} + \|\hat{\Sigma}\|_{op}), \|X\|_F + \|\hat{\Sigma}\|_F) }{\min_k(s_{n_k}^2 - s_{n_{k+1}}^2)}      \right)^2\|X-\hat{\Sigma}\|_{op}^2
    \]
\end{lemma}

\begin{remark}
To better understand $x$, consider the special case where each signal singular value is of multiplicity 1. Since matrix $X$ is "close" to diagonal matrix $\hat{\Sigma}$, matrix $U$ and matrix $V$ should also be close to identity along signal directions: $\abs{U_{kk}}\approx 1$, $\abs{V_{kk}}\approx 1$, $U_{kk}V_{kk}\approx 1$ for $1\leq k\leq K$. Quantity $x$ captures how much $\abs{U_{kk}},\abs{V_{kk}},U_{kk}V_{kk}$ deviate from 1.  
\end{remark}

\begin{proof} Let $X = USV^T$. Then
    $X^TX = VS^2V^T$ is a symmetric matrix. Let $V(:,k) = (v_{n_{k-1} + 1}, \ldots, v_{n_k})$ be the singular vectors of $X$ corresponding to $s_{n_{k-1} + 1}, \ldots, s_{n_k}$. There is some freedom to choose $\hat{V}$, but for simplicity we pick $\hat{V}(:,k) = (e_{n_{k-1} + 1}, \ldots, e_{n_k})$ where $e_i$ is the $i$-th coordinate vector. As a result, $\hat{V}(:,k)^TV(:,k) = V(k,k)$. Apply DK-$\sin\theta$ theorem to $V$ and $\hat{V}$, we see that 
    \[
    \|\sin\Theta(V(:,k),\hat{V}(:,k))\|_F \leq \frac{2\min(\sqrt{n_k - n_{k-1}}\|X^TX - \hat{\Sigma}^2\|_{op},\|X^TX - \hat{\Sigma}^2\|_{F}  )}{\min (s_{n_{k-1}}^2 - s_{n_{k-1} + 1}^2, s_{n_k}^2 - s_{n_k + 1}^2)} 
    \]
    Let $\sigma_p$ be the singular values of $V(k,k)$, for $1\leq p \leq n_k - n_{k-1}$. Here we are using the notation for sub-block of a big matrix. Then $\sin\Theta(V(:,k),\hat{V}(:,k))$ is the diagonal matrix, whose diagonal entries are given by $\sqrt{1-\sigma_p^2}$. So $\|\sin\Theta(V,\hat{V})\|_F^2 = \sum_{p}(1-\sigma^2_p) = tr[ I - V(k,k)^TV(k,k)]$.  Now observe that 
    \begin{align*}
        \|X^TX - \hat{\Sigma}^2\|_{op} \leq & (\|X\|_{op} + \|\hat{\Sigma}\|_{op}) \|X-\hat{\Sigma}\|_{op}
    \end{align*}
    \begin{align*}
        \|X^TX - \hat{\Sigma}^2\|_{F}  \leq & (\|X\|_{F} + \|\hat{\Sigma}\|_{F}) \|X-\hat{\Sigma}\|_{op}
    \end{align*}
    We conclude that 
    \[
    tr[I - V(k,k)^TV(k,k)] \leq \left( \frac{2\min(\sqrt{n_k - n_{k-1}}(\|X\|_{op} + \|\hat{\Sigma}\|_{op}),  \|X\|_{F} + \|\hat{\Sigma}\|_{F})}{\min (s_{n_{k-1}}^2 - s_{n_{k-1} + 1}^2, s_{n_k}^2 - s_{n_k + 1}^2)}\|X-\hat{\Sigma}\|_{op}\right)^2.
    \]
    Similar conclusion is true for $U(k,k)$. Next we bound $tr(I - U(k,k)V(k,k)^T)$.

    $$\|\hat{\Sigma}(k,k) - X(k,k)\|_F^2 = \sum_{q_1,q_2=n_{k-1}+1}^{n_k} (\hat{\Sigma}_{q_1q_2} - X_{q_1q_2})^2\leq (n_k - n_{k-1})\|\hat{\Sigma} - X\|_{op}^2.$$
    \begin{align*}
        \|\hat{\Sigma}(k,k) - X(k,k)\|_F^2 = & \|\hat{\Sigma}(k,k) - \sum_j U(k,j)S(j,j)V^T(j,k)\|_F^2\\
        =& \|\hat{\Sigma}(k,k) - \sum_{j\neq k} U(k,j)S(j,j)V(k,j)^T - U(k,k)S(k,k)V(k,k)^T\|_F^2\\
        \geq & \|\hat{\Sigma}(k,k) - U(k,k)\hat{\Sigma}(k,k)V(k,k)^T\|_F^2\\
        &- \|\sum_{j\neq k} U(k,j)S(j,j)V(k,j)^T\|_F^2\\
        &- \|U(k,k)(S(k,k) - \hat{\Sigma}(k,k))V(k,k)^T\|_F^2\\
        \geq& s_{n_k}^2\|I -U(k,k)V(k,k)^T\|_F^2\\
        &- \sum_{j\neq k}\|U(k,j)\|_F^2\|S(j,j)V(k,j)\|_{op}^2 \\
        &- \|X- \hat{\Sigma}\|_{op}^2\\
        \geq& s_{n_k}^2\|I -U(k,k)V(k,k)^T\|_F^2\\
        &- s_1^2tr(I - U(k,k)^TU(k,k)) tr(I - V(k,k)^TV(k,k))\\
        &- \|X- \hat{\Sigma}\|_{op}^2\\
    \end{align*}
    For every $n_k - n_{k-1} \times n_k - n_{k-1} $ matrix $M$, we have $tr(A) \leq \sum_p \abs{\lambda_p} \leq \sqrt{n_k -n_{k-1}} \left(\sum \abs{\lambda_p}^2\right)^{\frac{1}{2}} =\sqrt{n_k -n_{k-1}} \|M\|_F $. Using this inequality, we conclude that 
    \begin{align*}
        tr(I - U(k,k)V(k,k)^T) \leq &(n_k - n_{k-1})\frac{\|\hat{\Sigma} - X\|_{op}}{s_{n_k}} \\
        &+ \frac{s_1}{s_{n_k}}  \left( \frac{2\min(\sqrt{n_k - n_{k-1}}(\|X\|_{op} + \|\hat{\Sigma}\|_{op}),  \|X\|_{F} + \|\hat{\Sigma}\|_{F})}{\min (s_{n_{k-1}}^2 - s_{n_{k-1} + 1}^2, s_{n_k}^2 - s_{n_k + 1}^2)}\|X-\hat{\Sigma}\|_{op}\right)^2
    \end{align*}
    
    Summing on $k$, we conclude that 
    \[
    x\leq K\frac{\|\hat{\Sigma} - X\|_{op}}{s_K} + \frac{ms_1}{s_K}\left( \frac{  2\min(
    \sqrt{K}(\|X\|_{op} + \|\hat{\Sigma}\|_{op}), \|X\|_F + \|\hat{\Sigma}\|_F) }{\min_k(s_{n_k}^2 - s_{n_{k+1}}^2)}      \right)^2\|X-\hat{\Sigma}\|_{op}^2
    \]
\end{proof}
In the case where $d\to \infty$ , $m$ is a constant, $\alpha < 1$, $s_1 = O(s_K)$, $s_K > cs_{K+1}$, $K$ is a constant, $\|X - \hat{\Sigma}\|_{op} = o(\|\hat{\Sigma}\|_{op})$, and $\min_k (s_{n_k}^2 - s_{n_{k+1}}^2) \geq cs_K^2$ for some constant $c$, we have $x\leq O( \frac{\|X - \hat{\Sigma}\|_{op}}{\|\hat{\Sigma}\|_{op}})$.

\begin{lemma}
\label{lem: x to matrix}
    Let $X = USV^T$, and define 
$$x = 4K - \sum_{k: \text{signal}} \mathrm{tr}(U^T(k,k)U(k,k) + V^T(k,k)V(k,k) + 2U(k,k)V(k,k)^T).$$ Let $\Sigma$ be a $d\times d$ diagonal matrix whose diagonal entries are given by $b_1, \ldots, b_K, 0, \ldots, 0$,where $b_1= \ldots b_{n_1}>b_{n_1+1} = \ldots = b_{n_2} \geq \ldots = b_{n_m} = b_K$. Then 
\[
\|\Sigma - U^T\Sigma V\|_{op} \leq \|\Sigma\|_{op}((m+1)^3\sqrt{x} + 2(m+1)^3x)
\]
\end{lemma}
\begin{proof}
    First observe that $\|\Sigma - U^T\Sigma V\|_{op} \leq \sum_{i,j=1}^{m+1} \|\Sigma(i,j) - U^T\Sigma V(i,j)\|_{op}$. If $i\neq j$, then $\Sigma(i,j) = 0$, 
    \begin{align*}
        \|U^T\Sigma V(i,j)\|_{op} =& \| \sum_{\ell = 1}^{m+1} U(\ell,i)^T\Sigma(\ell,\ell)V(\ell,j)\|_{op}\\
        \leq & b_1 \sum_{\ell=1}^{m+1}\|U(\ell,i)\|_{op}\|V(\ell,j)\|_{op}\\
        \leq & b_1(m+1)\sqrt{x}. 
    \end{align*}
    If $i = j$, then $\Sigma(i,i) = b_{n_i}I$,
    \begin{align*}
        \|U^T\Sigma V(i,i) - b_{n_i}I\|_{op} \leq & \|b_{n_i}U(i,i)^T V(i,i) - b_{n_i}I\|_{op} + b_1\sum_{\ell \neq i} \|U(\ell,i)\|_{op}\|V(\ell,i)\|_{op}\\
        \leq& b_1(m+1)x + b_1(m+1)x.
    \end{align*}
    Therefore 
    \[
    \|\Sigma - U^T\Sigma V\|_{op} \leq \|\Sigma\|_{op}((m+1)^3\sqrt{x} + 2(m+1)^3x).
    \]
\end{proof}
This bound is not optimal in $m$, but throughout the paper, $m$ is of constant order and it is fine to miss a constant factor when estimating error.

\section{Proof of Theorem \ref{thm:mixed_dynamics}}
\label{sec: Proof of main thm 1}
\subsection{Weak bound}
We prove the following weak bound. 
\begin{proposition}
\label{prop: arthur weak bound}
For every $\varepsilon > 0$ and every $t < T$, we have with high probability,
    \begin{equation}
        \|W_1^TW_1 - \sqrt{A^TA + \sigma^4w^2 I} \|_{\mathrm{op}} \leq (1+\varepsilon)\sigma^2w
    \end{equation}
    Analogous results holds for $W_2$. 
\end{proposition}

Our main tool is the following lemma. 

\begin{lemma} 
\label{lem: conservation}
    For every cost $C$  and every time $t$, we have 
    \begin{equation}
        W_2^T W_2(t) - W_1W_1^T(t) = W_2^T(0)W_2(0) - W_1(0)W_1^T(0)
    \end{equation}
\end{lemma}

\begin{proof}

 Let $L = \|A^*-A\|^2$ be the loss function. Then \begin{align*}
    \partial_t W_1 &= 2W_2^T\nabla C \\
    \partial_t W_2 &= 2\nabla C W_1^T
\end{align*}
We see that \begin{equation*}
    \partial_t(W_1W_1^T) = 2W_2^T\nabla CW_1^T + W_1 \nabla C^TW_2 = \partial_t(W_2^TW_2)
\end{equation*}

\end{proof}Next, we show that at initialization, $ W_1W_1^T $ and $W_2^TW_2$ are approximately orthogonal projections, up to a factor. Stating precisely, we have the following lemma.
\begin{lemma}
\label{lem: proj approx}
    There exists two projections $P_1$ and $P_2: \mathbb{R}^{w}\to \mathbb{R}^{w}$ such that the following are true. 
    \begin{enumerate}
        \item The image of $P_1$ and $P_2$ are orthogonal to each other;
        \item With high probability, 
        \begin{equation}
            \|W_1W_1^T(0) - \sigma^2w P_1\|_{\mathrm{op}} =  O(\sigma^2\sqrt{wd}\log d)
        \end{equation} 
        \begin{equation}
            \|W_2^TW_2(0) - \sigma^2w P_2\|_{\mathrm{op}} =  O(\sigma^2\sqrt{wd}\log d)
        \end{equation}

    \end{enumerate}
    
\end{lemma}

\begin{proof}
[Proof of Lemma \ref{lem: proj approx}  ]:  At initialization, the rank of $w $is $w$ with probability 1. Therefore $ W_1W_1^T $ has $w - d$ eigenvalues that are 0, and the $w$ non-zero eigenvalues equals the eigenvalues of  $d\times d$ matrix $ W_1^TW_1 $. 

$ W_1^TW_1 (0)$ is a (scaled) Wishart ensemble, whose limiting distribution is given by the Marchenko-Pastur law. The Marchenko-Pastur law, as stated in \cite{alex2014isotropic}, proves the following. Let $X$ be an $M\times N$ matrix with complex-valued independent entries $X_{i\mu}$ such that 
\begin{enumerate}
    \item $\mathbb{E}[X_{i\mu}] = 0$;
    \item $\mathbb{E}[\abs{X_{i\mu}}^2] = \frac{1}{\sqrt{MN}}$;
    \item for every $p\in \mathbb{N}$, there exists a constant $C_p$ such that $$\mathbb{E}\left[\abs{(NM)^{\frac{1}{4}}X_{i\mu} }^p \right]\leq C_p.$$
\end{enumerate}
Here, $M$ satisfies $$0 < C^{-1} \leq \frac{\log M}{\log N} \leq C < \infty$$ for some constant $C $ independent of $M$ and $N$. Let $\phi = \frac{M}{N}$, which may or may not depend on $N$. Then the eigenvalues of $N\times N$ matrix $X^*X$ has the same asymptotics as 
\begin{equation*}
    \rho_\phi(dx):= \frac{\sqrt{\phi}}{2\pi}\sqrt{\frac{[(x-\gamma_{-})(\gamma_+ - x)]_+}{x^2}  }dx + (1-\phi)_+ \delta(dx) 
\end{equation*}
where 
\begin{equation}
    \gamma_{\pm} := \sqrt{\phi} + \frac{1}{\sqrt{\phi}} \pm 2
\end{equation}
In our situation, $W_1$ is a $w\times d$ matrix with independent and identically distributed Gaussian entries whose variance is $\sigma^2.$ Let $M = w$, $N = d$ and therefore $W_1^TW_1(0)$ has the same distribution as $\sigma^2\sqrt{wd}X^* X$. Notice that for this choice of $M$ and $N$, the asymptotic distribution eigenvalues of $X^*X$ is $\rho_{\frac{w}{d}}(dx)$. Notice that $\rho_{\frac{w}{d}}$ is supported on interval $\left[\sqrt{\frac{w}{d}} + \sqrt{\frac{d}{w}} - 2, \sqrt{\frac{w}{d}} + \sqrt{\frac{d}{w}} + 2\right]$, from which we conclude that in the limit, all eigenvalues of $W_1^TW_1(0)$ is approximately $\sqrt{\frac{d}{w}}$. \\
\\
By Theorem 2.10 of \cite{alex2014isotropic}, we have eigenvalue rigidity results for $X^*X$. Let $\lambda_k'$ be the $k$-th largest eigenvalue for $X^*X$. $\forall k \in \{1,2, \ldots, w\}$, we have \begin{equation}
    \abs{\lambda_k' - \gamma_k} < d^{-\frac{2}{3} + \varepsilon}
\end{equation}
with high probability. Here $\gamma_\alpha$ is defined through 
\begin{equation}
    \int_{\gamma_k}^\infty \rho_\phi(dx) = \frac{k}{d}
\end{equation}
Let $\lambda_k$ be the $k$-th largest eigenvalue of $ W_1^TW_1 (0)$. By the relationship between $ W_1^TW_1 (0)$ and $X^*X$, we know $\lambda_k$ has the same law as $\sigma^2\sqrt{wd}\lambda_k'$. Then for every $k$, 
\begin{align}
    \abs{\lambda_k - \sigma^2w} &\leq \sigma^2\sqrt{wd}\left(\abs{\lambda_k' -  \gamma_k} + \abs{\gamma_k - \sqrt{\frac{w}{d}} }\right) \\
    &\leq O(\sigma^2\sqrt{wd})
\end{align}
with high probability. We conclude that the first $w$ eigenvalues of $ W_1^TW_1 (0)$ is at most $O(\sigma^2\sqrt{wd})$-away from 1 and all other eigenvalues are 0. Therefore there exists a projection $P_1$ such that 
\begin{equation}
            \|W_1W_1^T(0) - \sigma^2w P_1\|_{\mathrm{op}} =  O(\sigma^2\sqrt{wd})
        \end{equation} 
Similarly, there exists a projection $\tilde{P}_2$ such that 
\begin{equation}
            \|W_2^TW_2(0) - \sigma^2w \tilde{P}_2\|_{\mathrm{op}} =  O(\sigma^2\sqrt{wd})
        \end{equation}
Notice that $\tilde{P}_2$ is not exactly orthogonal to $P_1$, and it remains to find a projection $P_2$ that is orthogonal to $P_1$ and is close to $W_2^TW_2(0)$. Assume that the column vectors of $W_1(0)$ are $u_1,\ldots,u_d\in \mathbb{R}^{w}$ and column vectors of $W_2^T(0)$ are $v_1, \ldots, v_d\in \mathbb{R}^{w}$. For $k = 1,2,\ldots, d$, we define vector $v_k'$ as 
\begin{equation}
    v_k' = v_k - P_1v_k.
\end{equation}
We claim that $P_1v_k$ is very small. By law of large numbers, $\|v_k\| \leq \sigma\sqrt{w} \log d$ with high probability. 
\begin{align}
    \|P_1v_k\| \leq &\|\frac{1}{\sigma^2w}W_1W_1^T(0) v_k\| + O(\sqrt{\frac{d}{w}}\|v_1\|)\\
    = & \frac{1}{\sigma^2w} \|\abra{u_1,v_k}u_1 + \ldots + \abra{u_w,v_k}u_w\| + O(\sigma\sqrt{d} \log d)
\end{align}

 Notice that $\abra{u_j,v_k}u_j, \forall j$ is a family of independent and identically distributed random vectors. For each of these random vectors, all entries have zero mean. The variance of any one of the entries is given by
 \begin{equation}
     \mathbb{E}\left[ \abra{u_i,v_k}^2\abra{u_j,e_\ell}^2 \right] = O(\sigma^6w)
 \end{equation}
We conclude that for each $\ell$ we have, by CLT,
\begin{equation}
    \frac{1}{\sigma^3w^{\frac{1}{2}}} \frac{\abra{u_1,v_k}\abra{u_1,e_\ell} + \ldots + \abra{u_w,v_k}\abra{u_w,e_\ell}}{\sqrt{d} } \xrightarrow{(d)}N(0,1)
\end{equation}
In particular,
\begin{align*}
    \mathbb{P}\{\max_{\ell} \sum_{i=1}^w \abs{\abra{u_i,v_k}\abra{u_k,e_\ell}} > 100\sigma^3\sqrt{d} w^{\frac{1}{2}}\log d\} \leq & w\mathbb{P}\{ \sum_{i=1}^w \abs{\abra{u_i,v_k}\abra{u_k,e_1}} > 100\sigma^3\sqrt{d} w^{\frac{1}{2} }\log d\}\\
    \leq & 2w\mathbb{P}\{N(0,1)> 100\log d\}\\
    \leq & O(d^{-50})
\end{align*}
Therefore with high probability, $\max_{\ell} \sum_{i=1}^w \abs{\abra{u_i,v_k}\abra{u_k,e_\ell}} \leq 100\sigma^3\sqrt{d} w^{\frac{1}{2} }\log d$.  This implies that with high probability, 
\begin{equation}
    \|\abra{u_1,v_k}u_1 + \ldots + \abra{u_w,v_k}u_w\| \leq \sqrt{w}\sigma^3\sqrt{d} w^{\frac{1}{2}} \log d
\end{equation}
\begin{equation}
    \|P_1v_k\|\leq \sigma\sqrt{d} \log d
\end{equation}
Now let ${W_2'}^{T}$ be the matrix with column vector $v_1',\ldots,v_w'$ and let $P_2$ be the projection to the column space of ${W_2'}^T$ . By construction $\|W_2 - W_2'\|_{\mathrm{op}}\leq O(\sigma\sqrt{d} \log d)$ , so $\|W_2^TW_2 - {W_2'}^TW_2'\|_{\mathrm{op}}\leq O(\sigma^2\sqrt{wd} \log d)$. Since the nonzero eigenvalues of $W_2^TW_2$ is at most $O(\sigma^2\sqrt{wd})$ from 1, we know that the nonzero singular values of ${W_2'}^TW_2'$ is also at most $O(\sigma^2\sqrt{wd})$ from 1. We conclude that $\|P_2 - W_2^TW_2\|_{\mathrm{op}}\leq O(\sigma^2\sqrt{wd}\log d)$. 
\end{proof}

\begin{proof}
[Proof of Proposition \ref{prop: arthur weak bound}]:  
We have 
\begin{align*}
    W_1^TW_1W_1^TW_1(t) 
    &= W_1^TW_2^TW_2W_1 + W_1^T(W_1W_1^T(t) - W_2^TW_2(t))W_1\\
    &= A^TA + W_1^T(W_1W_1^T(0) - W_2^TW_2(0))W_1 
\end{align*} From lemma \ref{lem: proj approx} we see that as positive semi-definite matrix, for every constant $\varepsilon > 0$,
\begin{equation}
    0\leq W_1W_1^T(0) \leq (1+\varepsilon)\sigma^2w I
\end{equation}
with high probability. Therefore 
\begin{equation}
    -(1+\varepsilon)\sigma^2w W_1^TW_1  \leq W_1^T(P_1(0) - P_2(0))W_1\leq (1+\varepsilon)\sigma^2w W_1^TW_1 .
\end{equation}
By moving terms around and corollary, we have 
\begin{align}\label{eqn: arthur weak 3}
    &( W_1^TW_1 )^2 - (1+\varepsilon)\sigma^2w W_1^TW_1  + (1+\varepsilon)^2\frac{\sigma^4w^2}{4}I\\
    \leq& A^TA + \frac{(1+\varepsilon)^2}{4}\sigma^4w^2 I \\
    \leq& ( W_1^TW_1 )^2 + (1+\varepsilon) W_1^TW_1  + (1+\varepsilon)^2\frac{\sigma^4w^2}{4}I
\end{align}
Theorem V.1.9 of \cite{bhatia2013matrix} states that the square-root function is operator monotone, which implies that if $A \geq B$ then $\sqrt{A} \geq \sqrt{B}$. Taking square-root, we have \begin{equation}
     W_1^TW_1  - (1+\varepsilon)\frac{\sigma^2w}{2}Id \leq \sqrt{A^TA + (1+\varepsilon)^2\frac{\sigma^4w^2}{4}} \leq  W_1^TW_1  + (1+\varepsilon)\frac{\sigma^2w}{2}
\end{equation} 
\end{proof}

\subsection{Strong Bound}
The weak bound does not provide useful information if $\|W_1^TW_1\|_{op}<<\sigma^2w$. For this reason we prove strong bound, which provide useful information if $\|W_1^TW_1\|_{op}<<\sigma^2w^{1+\Box}$ for some constant $\Box > 0$. Recall that the evolution of weight matrix in gradient descent is given by the following. 
\begin{equation}
    \frac{d}{dt}W_1(t) = \eta W_2^T\nabla C
\end{equation}
\begin{equation}
    \frac{d}{dt}W_2(t) = \eta \nabla C W_1^T
\end{equation}
The goal of this section is to prove that 
\begin{equation}
     W_1^TW_1  \approx \sqrt{A^TA + \sigma^4w^2I}
\end{equation}
For simplicity of notations, we shall assume that $C_1 =  W_1^TW_1 $, $C_2 = W_2W_2^T$, $\hat{C}_1 = \sqrt{A^TA + \sigma^4w^2I}$ and $C_2 = \sqrt{AA^T + \sigma^4w^2I}$. It is easy of see that $\hat{C}_1$ and $\hat{C}_2$ are invertible. Our main result for this section is the following proposition.
\begin{proposition}
\label{prop: isotropic c1 approx}
    For every cost $C$ we have 
    \begin{equation}
        \| W_1^TW_1  - \sqrt{A^TA + \sigma^4w^2 I} \|_{\mathrm{op}} \leq \min\{O(\sigma^2w),O\left(\sqrt{\frac{d}{w}}\| W_1^TW_1 \|_{op}\right)\}
    \end{equation}
\end{proposition}

\begin{proof}
    [Proof of Lemma \ref{prop: isotropic c1 approx} ]: We start from the equations:

\begin{align*}
W_1^T\left[(W_2^TW_2- W_1W_1^T )^2-\sigma^4w^2I\right]W_1 & = C_{1}^{3}-A^{T}AC_{1}- C_{1}A^{T}A-\sigma^4w^2C_{1}+A^{T}C_{2}A\\
W_2\left[(W_2^TW_2- W_1W_1^T )^2 - \sigma^4w^2I\right]W_2^T & = C_{2}^{3}-AA^{T}C_{2}-C_{2}AA^{T}-\sigma^4w^2C_{2}+AC_{1}A^{T}.
\end{align*}

Our goal is to show that $C_{1},C_{2}$ are close to the solution
$\hat{C}_{1}=\sqrt{A^{T}A+\sigma^4w^2I},\hat{C}_{2}=\sqrt{AA^{T}+\sigma^4w^2I}$ with 
\begin{align*}
0 & =\hat{C}_{1}^{3}-A^{T}A\hat{C}_{1}-\hat{C}_{1}A^{T}A-\sigma^4w^2\hat{C}_{1}+A^{T}\hat{C}_{2}A\\
0 & =\hat{C}_{2}^{3}-AA^{T}\hat{C}_{2}-\hat{C}_{2}AA^{T}-\sigma^4w^2\hat{C}_{2}+A\hat{C}_{1}A^{T}.
\end{align*}

Apriori, the cubic equation for $C_1$ and $C_2$ might have multiple solutions. We give an intuitive argument explaining why $\hat{C}_1$ and $\hat{C}_2$ are the correct solutions. By selecting a proper basis, we assume $A = diag(a_1,\ldots,a_d)$ is diagonal. Assume that $(W_2^TW_2 - W_1W_1^T)^2 = \sigma^2w I$. Also assume that $C_1$ and $C_2$ both commute with $A$. In this case the equations for $C_1$ and $C_2$ reduces to cubic equation for scalars. Let $\lambda_1, \ldots, \lambda_d$ be the eigenvalues of $C_1$. Solving the equations for scalars, we have $\lambda_i = 0$ or $\pm \sqrt{a_i^2 + \sigma^4w^2}$. By lemma \ref{lem: conservation} and lemma \ref{lem: proj approx}, we have $$W_1W_1^T(t) -W_2^TW_2 (t)= W_1W_1^T(0) - W_2^TW_2(0) = \sigma^2wP_1 - \sigma^2w P_2 + o(\sigma^2w).$$
Since $W_1W_1^T$ is positive semi-definite, all its eigenvalues are non-negative, and thus
\[
W_1W_1^T(t) \geq (\sigma^2w+o(\sigma^2w))P_1 .
\]
Since the top $w$ eigenvalues of $W_1^TW_1$ is the same as the top $w$ eigenvalues of $W_1W_1^T$, we conclude that $\lambda_i \geq \sigma^2w(1+o(1))$. This forces $\lambda_i = \sqrt{a_i^2 + \sigma^4w^2}$. 

Since $C_1$, $C_2$ are not assumed to be aligned with $A$, we cannot reduce the matrix cubic equations into scalar cubic equations. The high level idea for proving $dC_i:= C_i - \hat{C}_i$ is small is the inverse function theorem. 
\begin{enumerate}
    \item Step 1: show that LHS of the equations are small. 
    \item Step 2: reduce the RHS of the equations to a linear function of $dC_1$ and $dC_2$. 
    The system of equations is thus reduced to 
    \[
    \begin{pmatrix}
    \text{small}\\
    \text{small}
    \end{pmatrix} = 
    \begin{pmatrix}
        * & *\\
        * & *
    \end{pmatrix}
    \begin{pmatrix}
        v_i^TdC_1\\
        u_i^TdC_2
    \end{pmatrix}.
    \]
    The $*$ matrix is now the "Jacobian" matrix, and $u_i,v_i$ are left and right singular vectors of $A$.
    \item Step 3: prove that the "Jacobian" matrix $\begin{pmatrix}
        *&*\\
        *&*
    \end{pmatrix}$ is strictly positive definite, thus proving that $v_i^TdC_1$ and $u_i^TdC_2$ have small magnitude for all $i$. 
\end{enumerate}

For gradient flow, $W_2^TW_2 -  W_1W_1^T $ is preserved and there exists projections $P_1$ and $P_2$ such that $W_2^TW_2 -  W_1W_1^T  = \sigma^2w(P_2-P_1) + O(\sigma^2\sqrt{wd})$. Therefore for every unit vector $v$ we have
\begin{align*}
    \|v^TW_1^T\left[(W_2^TW_2- W_1W_1^T )^2-\sigma^4w^2I\right]W_1\|\leq \|C_1\|_{op}\sigma^4d^{\frac{1}{2}}w^{\frac{3}{2}}.
\end{align*}

Substracting the second pair of equations from the first pair and
denoting $dC_{i}=C_{i}-\hat{C}$, we obtain:
\begin{align}
\|C_1\|_{op}O(\sigma^4w^2\sqrt{\frac{d}{w}}) & =C_{1}^{3}-\hat{C}_{1}^{3}-A^{T}AdC_{1}-dC_{1}A^{T}A-\sigma^4w^2dC_{1}+A^{T}dC_{2}A  \label{eqn: C_1}\\
\|C_2\|_{op}O(\sigma^4w^2\sqrt{\frac{d}{w}}) & =C_{2}^{3}-\hat{C}_{2}^{3}-AA^{T}dC_{2}-dC_{2}AA^{T}-\sigma^4w^2dC_{2}+AdC_{1}A^{T}.\label{eqn: C_2}
\end{align}
Now since 
\begin{align*}
C_{1}^{3}-\hat{C}_{1}^{3} & =\hat{C}_{1}^{2}dC_{1}+\hat{C}_{1}dC_{1}C_{1}+dC_{1}C_{1}^{2},
\end{align*}
we substitute the above relation to equation \ref{eqn: C_1} and obtain
\begin{align}
\|C_1\|_{op}O(\sigma^4w^2\sqrt{\frac{d}{w}}) & =\left(\hat{C}_{1}^{2}-A^{T}A\right)dC_{1}+\hat{C}_{1}dC_{1}C_{1}+dC_{1}\left(C_{1}^{2}-A^{T}A\right)-dC_{1}+A^{T}dC_{2}A\\
 & =\hat{C}_{1}dC_{1}C_{1}+dC_{1}\left(C_{1}^{2}-A^{T}A\right)+A^{T}dC_{2}A, \label{eqn: u_1dC_1}
\end{align}
and similarly for equation \ref{eqn: C_2}. For any singular value $s_i$ of $A$, with left and right singular
vectors $u_{i},v_{i}$, we multiply equation \ref{eqn: u_1dC_1} to the left
by $v_{i}^{T}$, and divide both sides by $\sigma^2w$, to obtain an equation for $v_i^TdC_1$. Similarly, we obtain an equation for $u_i^TdC_2$: 
for $v_{i}^{T}dC_{1}$ and $u_{i}^{T}dC_{2}$:
\begin{align*}
 \|C_1\|_{op}O(\sigma^2\sqrt{wd})& =v_{i}^{T}dC_{1}\left(\sqrt{\left( \frac{s_i}{\sigma^2w} \right)^{2}+1}C_{1}+ \frac{1}{\sigma^2w}(C_{1}^{2}-A^{T}A)\right)+\left( \frac{s_i}{\sigma^2w} \right)u_{i}^{T}dC_{2}A\\
\|C_2\|_{op}O(\sigma^2\sqrt{wd}) & =u_{i}^{T}dC_{2,i}\left(\sqrt{\left( \frac{s_i}{\sigma^2w} \right)^{2}+1}C_{2}+\frac{1}{\sigma^2w}(C_{2}^{2}-AA^{T})\right)+\left( \frac{s_i}{\sigma^2w} \right)v_{i}^{T}dC_{1}A^{T}.
\end{align*}
Notice that $C_1^2-A^TA = W_1^T( W_1W_1^T -W_2^TW_2)W_1 = \sigma^2w W_1^T(P_1-P_2)W_1 + \|C_1\| O(\sigma^2\sqrt{wd})$. In the two equations above, by replacing $C_1^2-A^TA$ with $\sigma^2w W_1^T(P_1-P_2)W_1$, we are making an error of at most $\|C_1\|_{op}\|dC_1\|_{op}O(\sqrt{\frac{d}{w}})$. From weak bound we know that $\|dC_1\|_{op}\leq O(\sigma^2w)$. Therefore the error we made by making the approximation on the right hand side can be absorbed into left hand side.

To show that $\|v_{i}^{T}dC_{1}\|$ and $\|u_{i}^{T}dC_{2}\|$ are small, it suffices to show that the $(d_{in}+d_{out})\times(d_{in}+d_{out})$ block matrix
\[
\left(\begin{array}{cc}
\sqrt{\left( \frac{s_i}{\sigma^2w} \right)^{2}+1}C_{1}+W_1^T(P_1-P_2)W_1 & \left( \frac{s_i}{\sigma^2w} \right)A\\
\left( \frac{s_i}{\sigma^2w} \right)A^{T} & \sqrt{\left( \frac{s_i}{\sigma^2w} \right)^{2}+1}C_{2}+W_2(P_2-P_1)W_2^T
\end{array}\right)
\]
is strictly positive definite. This matrix can be further simplified to 
\[
\left(\begin{array}{cc}
W_{1}^{T} & 0\\
0 & W_{2}
\end{array}\right)\left(\begin{array}{cc}
\sqrt{\left( \frac{s_i}{\sigma^2w} \right)^{2}+1}I+P_{1}-P_{2} & \left( \frac{s_i}{\sigma^2w} \right)I\\
\left( \frac{s_i}{\sigma^2w} \right)I & \sqrt{\left( \frac{s_i}{\sigma^2w} \right)^{2}+1}I+P_{2}-P_{1}
\end{array}\right)\left(\begin{array}{cc}
W_{1} & 0\\
0 & W_{2}^{T}
\end{array}\right)
\]
The inner matrix can
then be rewritten as $RR^T$ where $R$ is defined as
\[ R=
\left(\begin{array}{c}
\sqrt{\sqrt{\left( \frac{s_i}{\sigma^2w} \right)^{2}+1}+1}P_{1}+\sqrt{\sqrt{\left( \frac{s_i}{\sigma^2w} \right)^{2}+1}-1}P_{2}\\
\sqrt{\sqrt{\left( \frac{s_i}{\sigma^2w} \right)^{2}+1}-1}P_{1}+\sqrt{\sqrt{\left( \frac{s_i}{\sigma^2w} \right)^{2}+1}+1}P_{2}
\end{array}\right).
\]
Let $Q =\left(\begin{array}{cc}
W_{1}^{T} & 0\\
0 & W_{2}
\end{array}\right) $.The "Jacobian" matrix is then $QR(QR)^T$. As described in the strategy, we need to show that the singular values are strictly positive. The smallest nonzero singular value of $QR(QR)^T$ is the same as the smallest nonzero singular value of $(QR)^TQR$. We expand $(QR)^TQR$ as follows.
\begin{align*}
 & (QR)^TQR\\
  =&\left(\sqrt{\left( \frac{s_i}{\sigma^2w} \right)^{2}+1}+1\right)P_{1}W_{1}W_{1}^{T}P_{1}+\left( \frac{s_i}{\sigma^2w} \right)P_{1}W_{1}W_{1}^{T}P_{2} \\
  &+ \left( \frac{s_i}{\sigma^2w} \right)P_{2}W_{1}W_{1}^{T}P_{1}+\left(\sqrt{\left( \frac{s_i}{\sigma^2w} \right)^{2}+1}-1\right)P_{2}W_{1}W_{1}^{T}P_{2}\\
 &+\left(\sqrt{\left( \frac{s_i}{\sigma^2w} \right)^{2}+1}+1\right)P_{2}W_{2}^{T}W_{2}P_{2}+\left( \frac{s_i}{\sigma^2w} \right)P_{1}W_{2}^{T}W_{2}P_{2}\\
 &+\left( \frac{s_i}{\sigma^2w} \right)P_{2}W_{2}^{T}W_{2}P_{1}+\left(\sqrt{\left( \frac{s_i}{\sigma^2w} \right)^{2}+1}-1\right)P_{1}W_{2}^{T}W_{2}P_{1}\\
  =&\left(\sqrt{\left( \frac{s_i}{\sigma^2w} \right)^{2}+1}+1\right)(P_{1}+P_{2})\sigma^2w\\
 & +2(\sqrt{\left( \frac{s_i}{\sigma^2w} \right)^{2}+1}-\left( \frac{s_i}{\sigma^2w} \right))(P_{1}W_{2}^{T}W_{2}P_{1}+P_{2}W_{1}W_{1}^{T}P_{2})\\
 &+\left( \frac{s_i}{\sigma^2w} \right)(W_{1}W_{1}^{T}+W_{2}^{T}W_{2}-\sigma^2w P_{1}-\sigma^2w P_{2}) +O(\sigma^2\sqrt{wd})\\
  =&\left(\sqrt{\left( \frac{s_i}{\sigma^2w} \right)^{2}+1}+1-\left( \frac{s_i}{\sigma^2w} \right)\right)(P_{1}+P_{2})\sigma^2w\\
  &+2(\sqrt{\left( \frac{s_i}{\sigma^2w} \right)^{2}+1}-\left( \frac{s_i}{\sigma^2w} \right))(P_{1}W_{2}^{T}W_{2}P_{1}+P_{2}W_{1}W_{1}^{T}P_{2})\\
  &+\left( \frac{s_i}{\sigma^2w} \right)(W_{1}W_{1}^{T}+W_{2}^{T}W_{2}) + O(\sigma^2\sqrt{wd}).
\end{align*}
where we used the fact that
\[
W_{1}W_{1}^{T}=P_{1}+P_{1}W_{2}^{T}W_{2}P_{1}+P_{1}W_{1}W_{1}^{T}P_{2}+P_{2}W_{1}W_{1}^{T}P_{1}+P_{2}W_{1}W_{1}^{T}P_{2}.
\]
The $(d_{in}+d_{out})$-th eigenvalue of the above is lower bounded
by $\sigma^2w\sqrt{\left( \frac{s_i}{\sigma^2w} \right)^{2}+1}+\sigma^2w-\sigma^2w\left( \frac{s_i}{\sigma^2w}\right)\geq\sigma^2w$. We conclude that 
\[
\|dC_1\|_{op} + \|dC_2\|_{op}\leq O(\sqrt{\frac{d}{w}})\|C_1\|_{op}, 
\]
\end{proof}

Compared to lemma \ref{lem: proj approx}, lemma \ref{prop: isotropic c1 approx} gives a tighter bound on $\|C_1 - \hat{C}_1\|_{\mathrm{op}}$ when $\|C_1\|_{\mathrm{op}} \leq \sigma^2w \sqrt{\frac{d}{w}}$. \\
\\

As suggested by an anonymous referee, it is possible to obtain the same approximated dynamics of $A_{\theta(t)}$ by imposing a non-homogeneous balance condition (in a different setup). Assume that $W_1W_1^T - W_2^TW_2 = 2\sigma^2w I$. Then \[
C_1^2 + 2\sigma ^2w C_1 - A^TA = 0;
\]\[
C_2^2 - 2\sigma^2w C_2 -AA^T = 0.
\]
Therefore $C_1 = -\sigma^2w + \sqrt{A^TA + \sigma^4w^2}$ and $C_2 = \sigma^2w + \sqrt{AA^T + \sigma^4w^2}$. Substituting into Gradient Flow equation, we have
\[
\frac{dA}{dt} = -\eta (\sqrt{AA^T + \sigma^4w^2}\nabla C + \nabla C \sqrt{A^TA + \sigma^4w^2}).
\]
The advantage of the setup is that it significantly simplifies the proof. The The limitation of the setup is that $W_1W_1^T - W_2^TW_2 \neq \sigma^4w^2I$ if the initial variance of entries of $W_1$ and $W_2$ are comparable. In this case, $W_1W_1^T - W_2^TW_2$ will have $w$ positive singular values and $w$ negative singular values, and the absolute value of positive and negative singular values are comparable. In the setup of our problem, the variance of entries of $W_1$ and $W_2$ equal.

\section{Gradient Flow Dynamics of $A_t$ in Active Regime}
\label{sec: GF for A_t}
\subsection{Saddle to Saddle Regime}

Let $A_t$ have the following dynamics: 
\[
d^2\frac{d}{dt}A_t = (A^*-A_t)\sqrt{A_t^TA_t+ \sigma^4w^2I} + \sqrt{A_tA_t^T + \sigma^4w^2I}(A^*-A_t).
\]
The goal of this section is to prove that the singular vectors of $A_t$ is well-aligned with the singular vectors of $A^*$, throughout the Saddle-to-Saddle regime. In the rest of the section, we will assume the dependence of $A_t$ on $t$ and use $A$ to represent $A_t$. If at initialization, $A_t$ commutes with $A^*$, then throughout the training, $A_t$ will always commute with $A^*$. In this section, we use a delicate stability argument to show that if $A_t$ almost commute with $A^*$ at the beginning of the Saddle to Saddle regime, then it will continue to be almost commutative with $A^*$ throughout the training process.

\begin{definition}
    Define $P_1$ be the family of $d\times d$ matrices $A$ that satisfies the following conditions. 
\begin{itemize}
    \item $s_K\geq C\sigma^2w$, $s_{K+1}\leq C'\sigma^2w$ and $\frac{s_{K+1}}{s_K} \leq c < \frac{1}{2}$ for some $d$-independent constants $c$, $C$ and $C'$. 
    \item If $a_k > a_{k+1}$, then $s_k - s_{k+1}\geq cs_k$ for some $d$-independent constant $c$.  
\end{itemize}
Define $P_1'$ to be the family of $w\times w$ matrix $A$ such that $s_k - s_{k+1} \geq\frac{c}{2}s_k$ if $a_{k+1}< a_k$, $\frac{s_{K+1}}{s_K} \leq \frac{3}{4}$ and $s_K \geq c\sigma^2w$.  Let $\gamma > 0$ be constant. Define $P_2(C,\gamma)$ be the family of $w\times w$ matrices $A$ satisfying the following conditions. 
\begin{itemize}
    \item (alignment of signals). Let $A = USV^T$. Define $$x = 4K - \sum_{k: \text{signal}} \mathrm{tr}(U^T(k,k)U(k,k) + V^T(k,k)V(k,k) + 2U(k,k)V(k,k)^T).$$ $P_2(C,\gamma)$ is the family of matrix $A$ such that $x\leq C d^{-\gamma}$.
\end{itemize}
\end{definition}

\begin{theorem}
    Assume that 
    \[
    d^2\frac{dA}{dt} =  (A^*-A) \sqrt{A^TA + \sigma^4w^2I} + \sqrt{AA^T + \sigma^4w^2I} (A^*-A) \]
    \[
    A(0) \in P_1\cap P_2(C_4,\gamma)
    \]
     Let $T = O(d\log d)$. Then $\forall t\in [0,T]$, we have 
    $A_t\in P_1\cap P_2(C,\min(1,\gamma))$. 
\end{theorem}
\begin{proof}
    A simple result of the induction lemma \ref{lem: induction lemma}, lemma \ref{lem: short time gap}, lemma \ref{lem: x dynamics} and lemma \ref{lem: long time gap}. 
\end{proof}
We will use the following induction lemma to show that the singular vectors of $A_t$ are roughly aligned with $A^*$. 
\begin{lemma}
\label{lem: induction lemma}
    Assume that $P_1$ and $P_2$ be families of increasing sets, and let $P_1'\supset P_1$. Assume that we have a family of matrices $A_t$, $0\leq t\leq T$ for some fixed number $T$. $T$ does not depend on the family of matrices. Assume that $A_0\in P_1\cap P_2$. Let $A_{[t_1,t_2]} = \{A_t: t_1\leq t\leq t_2\}$. Assume the following are true.
    \begin{enumerate}
        \item If $A_{[0,t]}\in P_1$ then there exists a constant $\varepsilon >0$ independent of $A_t$ such that $A_{[0,t+\varepsilon]}\subset P_1'$. 
    \item  Let $t_1 < t_2$. If $A_{[0,t_2]}\subset P_1'$ and $A_{[0,t_1]}\in P_2$, then $A_{[0,t_2]}\subset P_2$. 
        \item Let $t_1<t_2$. If $A_{[0,t_2]}\subset P_2$ and $A_{[0,t_1]}\in P_1$ then $A_{[0,t_2]}\subset P_1$. 
    \end{enumerate}
    Then $A_t\in P_1 $ and $A_t\in P_2$, $\forall 0\leq t\leq T$.
\end{lemma}
\begin{proof}
    Since $A_0$ satisfies $P_1$, use condition 1 we have $A_{[0,\varepsilon]} \subset P_1'$. Using condition 2, we know that $A_{[0,\varepsilon]} \subset P_2$ . By condition 3 we know that $A_{[0,\varepsilon}] \subset P_1$  and by condition 2, $A_{[0,\varepsilon]}\subset P_2$. Iterate the argument.
\end{proof}

\begin{lemma}
\label{lem: short time gap}
    If $A_t\in P_1$, then there exists $\tau = \sigma^2\sqrt{dw} d^{-1}$ such that $A_{[t,t+\tau]}\subset P_1'$.
\end{lemma}

\begin{proof} 

Let $0\leq s\leq \tau$. Using the approximation in time $[t,s+t]$ we have 
    \begin{align*}
        d^2\frac{dS}{dt} =& I\odot \left( U^TA^*V\sqrt{S^2 + \sigma^4w^2} + \sqrt{S^2 + \sigma^4w^2}U^TA^*V - 2S\sqrt{S^2 + \sigma^4w^2}\right) \\
    \end{align*}
    For $j = 1,2, \ldots, K+1$,
    \begin{align*}
        d^2\abs{\frac{ds_j}{dt}}\leq Cd(s_j+\sigma^2w)
    \end{align*}
    for some constant $C$ independent of $w$ and $A_t$. The conclusion follows from Gr\"onwall.
 
\end{proof}
\begin{lemma}
    Assume that 
    \begin{align*}
        d^2\frac{dA}{dt} =& (A^*-A+E)\sqrt{A^TA + \sigma^2wI} + \sqrt{AA^T + \sigma^4w^2I}(A^*-A+E).
     \end{align*}
     Then for signal $k$, 
     \[
     d^2\frac{d}{dt}  tr[U^T(k,k)U(k,k)] = tr[U^T(k,k)\sum_{j\neq k} U(k,j)D(j,k)] 
     \]
     
     Here for $j\neq k$,
     \[
     D(j,k) = R(k,j)\odot C(j,k)-R(j,k)\odot C(k,j)^T
     \]
     with $C = U^T(A^*+E)V$, $R_{pq}= \frac{s_p\tilde{s}_p + s_p\tilde{s}_q}{s_p^2-s_q^2}$ for $1\leq p,q\leq w$, $p\neq q$, and $R(k,j) = R_{n_k:n_{k+1}, n_j:n_{j+1} }$ being the $k,j$-th block of $R$.    
\end{lemma}
\begin{proof} Use SVD dervative.
    \begin{align}
    \frac{dU}{dt} =& U\left(F\odot \left[ U^T\frac{dA}{dt}VS + SV^T\frac{dA^T}{dt}U\right]\right)\\
    =& U\left( F\odot\left[ (U^T(A^*+E)V-S)\sqrt{S^2 + \sigma^4w^2}S + \sqrt{S^2 + \sigma^4w^2}(U^T(A+E)^*V-S)S \right]\right)\\
    &+ U\left( F\odot\left[ \sqrt{S^2 + \sigma^4w^2}S(V^T(A^*+E)^{T}U-S) + S(V^T(A^* + E)^{T}U-S)\sqrt{S^2 + \sigma^4w^2} \right]\right)\\
    =& U\left( F\odot\left[ U^T(A^*+E)V\sqrt{S^2 + \sigma^4w^2}S + \sqrt{S^2 + \sigma^4w^2}U^T(A^*+E)VS \right]\right)\\
    &+ U\left( F\odot\left[ \sqrt{S^2 + \sigma^4w^2}SV^T(A^*+E)^{T}U + SV^T(A^*+E)^{T}U \sqrt{S^2 + \sigma^4w^2} \right]\right)
\end{align}
Let $D = F\odot \left[ U^T\frac{dA}{dt}VS + SV^T\frac{dA^T}{dt}U\right]$. Since $F$ is anti-symmetric and the term in square bracket is symmetric, we know $D$ is anti-symmetric. Then  $\frac{dU}{dt} = UD$. Let $\tilde{S} = \sqrt{S^2 + \sigma^4w^2}$, $C = U^T(A^*+E)V$. As a result, $\forall 1\leq p,q \leq w$,
\begin{align*}
    \frac{d  U_{p q}}{dt} =& \sum_{r:r\neq q}   U_{p r}\frac{1}{s_q^2-s_r^2}\left[ (s_q\tilde{s}_q + \tilde{s_r}s_q)C_{rq} + (\tilde{s_r}s_r + s_r\tilde{s}_k)C_{rq}\right]
\end{align*}
Let $R_{pq} = \frac{s_p\tilde{s}_p + s_p\tilde{s}_q}{s_p^2-s_q^2}$ if $p\neq q$ and $R_{pp} = 0$.  
Then \begin{align*}
    \frac{d  U}{dt}(k,k) = \sum_{j:j\neq k}   U(k,j)\left(R(k,j)\odot C (j,k) - R(j,k)\odot  C(k,j)^T\right)
\end{align*}

$$D(j,k) = R(k,j)\odot C(j,k)- R(j,k)\odot C(k,j)^T.$$
 As a result, for $n_k\leq p,q < n_{k+1}$,
\begin{align*}
    \frac{1}{2}\frac{d}{dt}tr[U^T(k,k)U(k,k)] =& \sum_{p,q\in [n_k,n_{k+1})} U_{pq} \sum_r U_{pr}D_{rq}\\
    =& \sum_{p,q,r\in [n_k,n_{k+1}) } U_{pq} U_{pr}D_{rq} + \sum_{p,q\in [n_k,n_{k+1})} U_{pq} \sum_{r\notin [n_k,n_{k+1})} U_{pr}D_{rq}\\
    =& \sum_{p,q\in [n_k,n_{k+1})} U_{pq} \sum_{r\notin [n_k,n_{k+1})} U_{pr}D_{rq}\\
    =& tr[U^T(k,k)\sum_{j\neq k} U(k,j)D(j,k)]
\end{align*}

\end{proof}

\begin{lemma}
\label{lem: x dynamics}
    If $A_{[t_1,t_2]}\subset P_1$ and $A_{t_1}\in P_2(C_4,\gamma)$, then 

    \[
    d^2\frac{dx}{dt} \leq -cdx + O(\sqrt{dx}),
    \]
     where $\forall c =   \inf_{t\in [t_1,t_2]}  \min( \frac{(s_k^*-s_j^*)(s_k+s_j)(\tilde{s}_k + \tilde{s}_j)}{s_k^2-s_j^2}, \frac{(s_k-s_j)(s_k^*+s_j^*)(\tilde{s}_k + \tilde{s}_j)}{s_k^2-s_j^2} ) . $ In particular, $A_{[t_1,t_2]}\subset P_2(C, \min(\gamma,1))$ for some constant $C$. If $\gamma < 1$ then we can take $C = C_4$. 
\end{lemma}

\begin{proof}
Use previous lemma. Let $$x = 4K - \sum_{k: \text{signal}} \mathrm{tr}(U^T(k,k)U(k,k) + V^T(k,k)V(k,k) + 2U(k,k)V(k,k)^T).$$ $x$ measures the alignment of singular vectors of $A$ with singular vectors of $A^*$. From $UU^T = I$ we have 
\[
\sum_\ell U(k,\ell)U(k,\ell)^T = I
\]
and for every $j \neq k$, 
\[
U(k,j)U(k,j)^T \leq I - U(k,k)U(k,k)^T\leq O(x)
\]
In particular, $\|U(k,j)\|_{op}\leq O(\sqrt{x})$.  We first estimate $C$. For all $j = 1,2,\ldots,w$, $a_j\neq a_k$, 
\begin{align*}
    (U^TA^*V)(j,k) =& U^T (j,j)S^*(j,j)   V(j,k) + U^T (j,k)S^*(k,k)   V(k,k) + \sum_{\ell\neq j,k} U^T (j,\ell)S_\ell^*   V(\ell,k)  \\
    =& U^T (j,j)S^*(j,j)   V(j,k) + U (k,j)^TS^*(k,k)   V(k,k)\\
    &+O(dx) 
\end{align*}
\begin{align*}
    (U^TEV)(j,k) =& \sum_{\ell_2 \neq k} U^T(j,j)E(j,\ell_2)V(\ell_2,k) + \sum_{\ell_1\neq j} U^T(j,\ell_1)E(\ell_1,k)V(k,k) \\
    &+ U^T(j,j)E(j,k)V(k,k) + \sum_{\ell_1\neq j,\ell_2\neq k} U^T(j,\ell_1)E(\ell_1,\ell_2)V(\ell_2,k)\\
    =& O(\sqrt{d} \sqrt{x}) + O(\sqrt{d} ) + O(\sqrt{d} x)\\
    =& O(\sqrt{d} )
\end{align*}
\begin{align*}
    D(j,k) =& R(k,j)\odot C(j,k) - R(j,k)\odot C(k,j)^T\\
    =& R(k,j)\odot {(U(j,j)^TS^*(j,j)V(j,k))} + R(k,j)\odot U(k,j)^TS^*(k,k)V(k,k)\\
    &-R(j,k)\odot(V(k,j)^TS^*(k,k)U(k,k) + V(j,j)^TS^*(j,j)U(j,k))\\
    &+ O(dx+ \sqrt{d} ) 
\end{align*}

By $UU^T = I$ we know that $\sum_\ell U(j,\ell)U^T(\ell,k) = 0$. Therefore 
\[
U(j,j) U(k,j)^T + U(j,k)U(k,k)^T + O(x) = 0
\]
\[
U(j,k) = -O(x) -U(j,j)U(k,j)^TU(k,k).
\]
Similarly we have for $V$, 
\[
V(k,j) = -O(x) - V(k,k)V(j,k)^TV(j,j).
\]

We can rewrite $D(j,k)$ as 
\begin{align*}
    D(j,k) =& R(k,j)\odot(U(j,j)^TS^*(j,j)V(j,k) + U(k,j)^TS^*(k,k)V(k,k))   \\
    +&  R(j,k)\odot(V(j,j)^TV(j,k)V(k,k)^T S^*(k,k)U(k,k) + V(j,j)^TS^*(j,j)U(j,j)U(k,j)^TU(k,k))                  \\
    +& O(dx+ \sqrt{d} )  
\end{align*}
\begin{align*}
    &d^2\frac{d}{dt}\mathrm{tr}\left[ U(k,k)^TU(k,k)\right] \\
    =& 2\mathrm{tr}\left[ U(k,k)^T\sum_jU(k,j)D(j,k) \right]\\
    =& 2\sum_j \mathrm{tr}\left [ U(k,k)^TU(k,j)\left( R(k,j)\odot(U(j,j)^TS^*(j,j)V(j,k)  + U(k,j)^TS^*(k,k)V(k,k))  \right)  \right]\\
    &+ 2\sum_j \mathrm{tr}\left [ U(k,k)^TU(k,j)\left( R(j,k)\odot(V(j,j)^T V(j,k)V(k,k)^TS^*(k,k)U(k,k) )   \right)  \right]\\
    &+ 2\sum_j \mathrm{tr}\left [ U(k,k)^TU(k,j)\left( R(j,k)\odot( V(j,j)^TS^*(j,j)U(j,j)U(k,j)^TU(k,k))   \right)  \right]\\
    &+ O(dx^{\frac{3}{2}} + \sqrt{d} \sqrt{x}) \\
    \geq & -2\sum_j \abs{S^*(j,j)R(k,j) + S^*(k,k)R(j,k)}\sqrt{tr(U(k,j)^TU(k,j)}\sqrt{tr(V(j,k)^TV(j,k)}\\
    &+ 2\sum_jS^*(k,k)\mathrm{tr}[U(k,k)^TU(k,j) R(k,j)\odot (U(k,j)^TU(k,k)) ]\\
    &+ 2\sum_jS^*(j,j)\mathrm{tr}[ U(k,k)^TU(k,j)R(j,k)\odot( 
       U(k,j)^TU(k,k) )]\\
    &+ O(dx^{\frac{3}{2}}) + O(\sqrt{dx}) 
\end{align*}
We know that 
\begin{align*}
   &(S^*(k,k)R(k,j)+S^*(j,j)R(j,k)) -\abs{S^*(j,j)R(k,j) + S^*(k,k)R(j,k)} \\
   =& \min( \frac{(s_k^*-s_j^*)(s_k+s_j)(\tilde{s}_k + \tilde{s}_j)}{s_k^2-s_j^2}, \frac{(s_k-s_j)(s_k^*+s_j^*)(\tilde{s}_k + \tilde{s}_j)}{s_k^2-s_j^2} ) \\
   \geq & cd
\end{align*}
for some positive constant $c$ independent of $w$. We conclude that 
\[
d^2\frac{d}{dt}tr[U(k,k)^TU(k,k)] \geq cd x + O(dx^{\frac{3}{2}}) + O(\sqrt{dx}).
\]
The same trick applies to $V$. We similarly have 
\[
d^2\frac{d}{dt}tr[V(k,k)^TV(k,k)] \geq  cd x + O(dx^{\frac{3}{2}}) + O(\sqrt{dx}).
\]
It remains to show that $\frac{d}{dt} tr[V(k,k)^TU(k,k)]$ bounded below.
\begin{align*}
    d^2\frac{d}{dt}\mathrm{tr}[U(k,k)^TV(k,k) ] = tr[U(k,k)^T\frac{dV}{dt}(k,k)] + tr[V(k,k)^T\frac{dU}{dt}(k,k)]
\end{align*}

\begin{align*}
    tr[V(k,k)^T\frac{d}{dt}U(k,k)] =& tr[U(k,k)^T\frac{d}{dt}U(k,k)] + tr[(V(k,k)^T - U(k,k)^T)\frac{d}{dt}U(k,k)]\\
    \geq& tr[U(k,k)^T\frac{d}{dt}U(k,k)] -\|V(k,k) - U(k,k)\|_F\|\frac{d}{dt}U(k,k)\|_F\\
    \geq& O(dx^{\frac{3}{2}}) +  cd x + O(dx^{\frac{3}{2}}) + O(\sqrt{dx})\\
    \geq&  cd x + O(dx^{\frac{3}{2}}) + O(\sqrt{dx})
\end{align*}

Similarly, 
\begin{align*}
    tr[U(k,k)^T\frac{d}{dt}V(k,k)]\geq  cd x + O(dx^{\frac{3}{2}}) + O(\sqrt{dx}) .
\end{align*}

This proves that 
\begin{align}
    d^2\frac{dx}{dt} \leq-  cd x + O(dx^{\frac{3}{2}}) + O(\sqrt{dx}).
\end{align}
Observe that if $x \geq d^{-1 + \varepsilon}$ for some $\varepsilon > 0$, then $\frac{dx}{dt}\leq -\frac{c}{2}dx$ and therefore $x$ decrease exponentially. On the other hand, if $x\leq d^{-1-\varepsilon}$ for some $\varepsilon > 0$, then $\frac{dx}{dt} \geq O(\sqrt{dx})$ and therefore $x$ might increase. We therefore conclude that $A_{[t_1,t_2]}\subset P_2(C,\min(\gamma,1))$.
\end{proof}

\begin{lemma} 
\label{lem: long time gap}
Assume that $A_{[0,t_2]}\subset P_2(C_4,\gamma)$ and $A_{[0,t_1]}\subset P_1$. Then $A_{[0,t_2]}\subset P_1$.
\end{lemma}
\begin{proof}  By assumption we know that $t_2 = O(d\log d)$. 
\begin{align*}
        d^2\frac{dS}{dt} =& I\odot \left( U^T(A^*+E)V\sqrt{S^2 + \sigma^4w^2} + \sqrt{S^2 + \sigma^4w^2}U^T(A^*+E)V - 2S\sqrt{S^2 + \sigma^4w^2}\right) 
    \end{align*}
    \begin{align*}
        d^2\frac{ds_{K+1}}{dt} = 2 \left(O(\sqrt{d}) + O(d^{1-\gamma}) -s_{K+1}\right)\sqrt{s_{K+1}^2+\sigma^4w^2}
    \end{align*}
   \[
   \abs{s_{K+1}(t_2) - s_{K+1}(t_1)} \leq O(\sigma^2w (d^{-\frac{1}{2}} + d^{-\gamma})\log d ) << \sigma^2w 
   \]
   It remains to verify the gap between different families. Let $T$ be the first time when $\|A\|_{op} = \frac{1}{2}s_K^*$ and assume that $t_2 \leq T$. Assume that $a_k > a_j$, then there exists some constant $\varepsilon$ such that 
\[
d^2\frac{ds_k}{dt} \geq (s_k^*(1-\varepsilon)-s_k)s_k
\]
\[
d^2\frac{ds_j}{dt} \leq (s_j^*(1+\varepsilon) - s_j)(s_j + \sigma^2w)
\]
An application of Gr\"onwall inequality on $s_k$ and $s_j$ implies that there exists some constant $\delta > 0$ that depends only on $a_j,a_k$ such that $\frac{s_j(T)}{s_k(T)} \leq O(d^{-\delta})$. Next we deal with the case when $t_2 > T$. One important observation is that the $\inf\{t>T: s_1(t) \geq \frac{a_kd + a_jd}{2}\} - T = O(d)$, as a simple result of Gr\"onwall. Moreover, $s_2(T+ O(d))\leq \frac{1}{4}s_K^*$. This implies that $s_1 - s_k > cs_1$ for some constant $c$. Repeat this argument for all remaining signal singular values completes the proof.  
\end{proof}

\subsection{First NTK Regime}
At initialization, the singular values of $A$ are of order $\sigma^2\sqrt{wd}$ by using Mechenko-Pastur law on $A^TA$, which is infinitely smaller than $\sigma^2 w$.  As a result, there is a very short period when the dynamics is very close to NTK. This section is again a delicate stability argument to show that at the end of the first NTK regime, $A_t$ is almost commutative with $A^*$. 
\begin{lemma}
\label{lem: At short NTK}
    Let $t_1$ be the first when the first singular value of $A$ hits $\sigma^2wd^{-\frac{\delta}{2}}$ for $\delta = \frac{\gamma_w -1}{2}$. Then $A_{t_1}\in  P_2(C,\min(\frac{1}{2},\frac{\delta}{2}))$ for some constant $C$. 
\end{lemma}
\begin{proof}
We approximate the dynamics with NTK dynamics. Assume that 
\[
d^2\frac{dB}{dt} = 2(A^*+E-B)\sigma^2w
\]
with $B(0) = A(0)$. Then
\begin{align*}
    d^2\frac{d(A-B)}{dt} =& 2(B-A)\sigma^2w (A^*+E-A) \\
    &+ (\sqrt{A^TA + \sigma^4 w^2} - \sigma^2w) + (\sqrt{AA^T + \sigma^4w^2} - \sigma^2 w)(A^*+E-A).
\end{align*}
A simple application of Gr\"onwall inequality implies that 
\begin{align*}
    \|A-B\|_{op}(t) \leq&  \frac{1}{d^2}e^{-2\sigma^2wt}\int_0^te^{2\sigma^2w\tau} \frac{2s_1(\tau)^2}{\sigma^2w}\|A^*+E-A\|_{op}(\tau) d\tau\\
    \leq & 3\frac{\|A_t\|^2_{op}}{\sigma^2w}s_1^*\frac{t}{d^2}
\end{align*}
We first check $P_2$. Clearly, $t_1$ is of order $d^{1-\frac{\delta}{2}}$. Then  we have 
\begin{align*}
    \|A_{t_1} - \frac{t_1}{d^2}\sigma^2w A^*\|_{op}\leq & 2d^{-\delta}\|A_{t_1}\|_{op}  + \|\frac{t_1}{d^2}\sigma^2wE\|_{op} +O(\sigma^2\sqrt{wd})\\
    =& (O(d^{-\delta} )+O(d^{-\frac{1}{2}}) + O(d^{-\frac{\delta}{2}}) )\|A_{t_1}\|_{op}
\end{align*}
 By lemma \ref{lem: short time gap} we have $x \leq Cd^{-\min(\frac{1}{2},\frac{\delta}{2})}$ for some constant $C$. 
\end{proof}
\begin{lemma}
\label{lem: A_theta NTK Final Alignment}
    Let $t_2 = cd$, where $c$ is a constant to be chosen. Then 
    \[
    A_{t_2}\in P_1\cap P_2(C, \min(\frac{\gamma_{w}-1}{4},\frac{1}{2})).
    \]
    Moreover, there exists constants $b_1= \ldots b_{n_1}>b_{n_1+1} = \ldots = b_{n_2} \geq \ldots = b_{n_m} = b_K$, such that for a $d\times d$ diagonal matrix $\Sigma$ whose diagonal entries are given by $b_1, \ldots, b_K, 0, \ldots, 0$, we have $\|A_{t_2} - \Sigma\|_{op} \leq O(\sigma^2w)d^{-\min(\frac{\gamma_w -1}{8}, \frac{1}{4}  )}$.
\end{lemma}

\begin{proof}
At time $t_2$, $\|B(t_2)\|_{op} = 2a_1c\sigma^2w(1+o(1))$. From previous lemma we know that 
\[
\|A -B\|_{op}(t) \leq 3\frac{\|A_t\|^2_{op}}{\sigma^2w}s_1^*\frac{t}{d^2}.
\]
Let $t = t_2$ we see that $\abs{\|A(t_2)\|_{op} - \|B(t_2)\|_{op}}\leq \|A-B\|_{op}(t_2)\leq 2\frac{\|A_{t_2}\|^2_{op}}{\sigma^2w}a_1c$. 
Therefore as long as $c$ is sufficiently small, we can guarantee that $\|A-B\|_{op}(t) \leq \frac{\min_{1\leq k\leq K} (a_{k-1} - a_k)}{100a_1}\|B\|_{op}(t)$. This guarantees that the singular values of $A$ grows linearly in $[t_1,t_2]$, that $s_k-s_{k+1} \geq cs_k$ if $a_k > a_{k+1}$ for some constant $c$, that $s_{K+1} <\frac{1}{2}s_K$ and that $A_{t_2}\in P_1$. It remains to check that $A_{t_2}$ satisfies $P_2(C,\frac{\gamma_{w}-1}{2})$. We are ready to use similar techniques as in lemma \ref{lem: x dynamics}. All computation are similar, and the only difference is that $\abs{R(k,j)}=O(\frac{\sigma^2w}{s_K})$. We have 
\[
d^2\frac{d}{dt} \mathrm{tr}[U(k,k)^TU(k,k)] \geq c\frac{\sigma^2w}{s_K}dx  + O(dx^{\frac{3}{2}}\frac{\sigma^2w}{s_K}) + O(\sqrt{dx}\frac{\sigma^2w}{s_K})
\]
\[
d^2\frac{dx}{dt}\leq  2C\frac{\sigma^2w}{a_K(\sigma^2wwt + \sigma^2wd^{-\frac{\delta}{2}} )}(-cdx + dx^{\frac{3}{2}} +\sqrt{dx})
\]
Therefore $A_{[t_1,t_2]}\subset P_2(C, \min(\frac{\gamma_{w}-1}{4},\frac{1}{2}))$. Let $\gamma = \min(\frac{\gamma_{w}-1}{4},\frac{1}{2}).$ From lemma \ref{lem: At short NTK} we see that if $n_{k_1} + 1\leq p_1,p_2 \leq n_k$, then $\abs{s_{p_1} - s_{p_2}}\leq O(d^{-\gamma}) s_{K} $. Moreover, the dynamics of $s_{p_i}, i = 1,2$ are both given by  
\[
d^2\frac{ds_{p_i}}{dt} = 2(O(\sqrt{d}) + O(d^{1-\gamma}) + a_kd - s_{p_i})\sqrt{s_{p_i}^2 + \sigma^4w^2}.
\]

\[
\frac{d\abs{s_{p_1} - s_{p_2}}}{dt} \leq O(d^{-1})\abs{s_{p_1} - s_{p_2}} 
\]
\[
\abs{s_{p_1} - s_{p_2}}(t_2)\leq O(1)\abs{s_{p_1} - s_{p_2}}(t_1) = \sigma^2 w O(d^{-\gamma - \frac{\gamma_w-1}{4}}).
\]
Let $b_{n_i} = s_i$. By lemma \ref{lem: x to matrix}, there exists a $d\times d$ diagonal matrix $\Sigma$, whose diagonal entries are given by $b_1 = \ldots = b_{n_1} > \ldots = b_{n_2} \geq \ldots = b_K$, such that $\|A_{t_2} - \Sigma \|_{op} \leq O(\sigma^2w)d^{-\frac{\gamma}{2}} .$
\end{proof}

\section{ The Gradient Flow Dynamics of $A_{\theta(t)}$ in Lazy Regime}
\label{sec: GF for A_theta in Lazy}

In this section, we use a stability argument to show that if $\sigma^2w $ is infinitely larger than $d$, then the algorithm cannot converge. 

\begin{proposition}
    With high probability, for all time $t$, $\|A_{\theta(t)} - A^*\|_F^2\geq \frac{1}{3} \min(\|A^*\|_F^2,\|E\|_F^2 )$.
\end{proposition}
\begin{proof}
    In gradient flow dynamics, $\|A_{\theta(t)} - A^*-E\|_F$ is decreasing with time and therefore for all time $t$, $\|A_{\theta(t)}\|_F^2\leq 9(\|A^*\|_F^2+\|E\|_F^2)$. In particular, $\|A_{\theta(t)}\|_{op}\leq 9(\|A^*\|_F+\|E\|_F)$. By assumption, $\|\hat{C}_1 - C_1\|_{op}\leq O(\sqrt{\frac{d}{w}}) \|C_1\|_{op} = O(d\sqrt{\frac{d}{w}})$, and the dynamics of $A_{\theta(t)}$ can be written as 
    \[
    d^2\frac{dA_{\theta(t)}}{dt} = 2(A^*+E-A) \sigma^2w+ (A^*+E-A)O(d\sqrt{\frac{d}{w}}) + O(d\sqrt{\frac{d}{w}})(A^*+E-A)
    \]
    Assume that $B(0) = A_{\theta(0)}$ and 
    \[
    d^2\frac{dB}{dt} = 2(A^*+E-B)\sigma^2w.
    \]
    Then 
    \begin{align*}
       \frac{1}{2}d^2 \frac{d}{dt}\|A_{\theta(t)}-B\|_F^2 = & -2\sigma^2w\|A_{\theta(t)}-B\|_F^2 + tr( (A-B)^T(A^*+E-A)O(d\sqrt{\frac{w}{d}}))\\ 
       & + tr( (A-B)^TO(d\sqrt{\frac{w}{d}})(A^*+E-A)) \\
       \leq &-2\sigma^2w\|A_{\theta(t)}-B\|_F^2 +\|A_{\theta(t)}-B\|_FO(d^2\sqrt{\frac{d}{w}})
    \end{align*}
    This implies that for all time, $$\|A_{\theta(t)}-B\|_F\leq dO(\frac{d}{\sigma^2w}\sqrt{\frac{d}{w}}).$$ 
    The dynamics of $B$ is linear, and we have 
    \[
    B_t = (A^*+E)(1-e^{-2\sigma^2wt/d^2}) + B_0 e^{-2\sigma^2wt/d^2}.
    \]
    \[
    \|B_t - A^*\|_F^2 = \|e^{-2\sigma^2t/d^2}A^* + (1-e^{-2\sigma^2wt/d^2})E+ B_0e^{-2\sigma^2wt/d^2}\|_F^2\geq 0.99 \|e^{-2\sigma^2t/d^2}A^* + (1-e^{-2\sigma^2wt/d^2})E\|_F^2
    \]
    Let $P$ be the projection to the image of $A^*$. Then with high probability,
    \begin{align*}
        \|B_t - A^*\|_F^2 = & \|e^{-2\sigma^2wt/d^2}A^* + (1-e^{-2\sigma^2wt/d^2})E+ B_0e^{-2\sigma^2wt/d^2}\|_F^2 \\
        \geq& 0.99 \|e^{-2\sigma^2t/d^2}A^* + (1-e^{-2\sigma^2wt/d^2})E\|_F^2\\
        = & 0.99 \|e^{-2\sigma^2wt/d^2}A^* + (1-e^{-2\sigma^2wt/d^2}) PE\|_F^2 + 0.99 \|(1-e^{-2\sigma^2wt/d^2})(I-P)E\|_F^2\\
        \geq & \frac{1}{2}\min(\|A^*\|_F^2,\|E\|_F^2 )
    \end{align*}
   Since $\|A_{\theta(t)} - B\|_F^2 = o(d^2)$, we have 
   \[
   \|A_{\theta(t)} - A^*\|_F^2 \geq \frac{1}{3} \min(\|A^*\|_F^2,\|E\|_F^2 )
   \]

\end{proof}

\section{The Gradient Flow Dynamics of $A_{\theta(t)}$ in Active Regime}
In section \ref{sec: GF for A_t}, we proved stability for gradient flow dynamics for $A_t$. In this section, we prove similar stability statements for $A_{\theta(t)}$, using the approximation results from section \ref{sec: Proof of main thm 1}. We also summarize the behavior of $A_{\theta(t)}$ in section \ref{subsection: dynamics summary}.
\label{sec: GF for A_theta in Active}
\subsection{Saddle-to-Saddle Regime}
The goal of this section is to show that at the end of the first NTK regime, the singular vectors of $A_{\theta(t)}$ are roughly aligned with $A^*$. Moreover, the alignment remains to be good throughout the mixed regime. The following generalization of lemma \ref{lem: x dynamics} and lemma \ref{lem: long time gap} will be useful.
\begin{lemma}
    Assume that $A'_{[t_1,t_2]}\subset P_1$ and $A_{t_1}\in P_2(C_4,\gamma)$. Let $s_1',\ldots, s_K'$ be the first $K$ singular values of $A'$. Assume that the dynamics of $A'$ is the following.
    \begin{align*}    
    d^2\frac{dA'}{dt} = (A^*+E-A') \sqrt{A^{'T}A' + \sigma^4w^2I} &+ \sqrt{A'A^{'T}+\sigma^4w^2}(A^*+E-A')\\
    &+ O(d^{-\lambda}s_K')(A^*+E-A') + (A^*+E-A')O(d^{-\lambda}s_K')
    \end{align*}
    
    Here, $O(d^{-\lambda}s_K')$ is a matrix whose operator norm is bounded by $O(d^{-\lambda}s_K')$. Then $A_{[t_1,t_2]}\subset P_2(C_4,\min(\gamma ,1 , 2\lambda))$. 
\end{lemma}
\begin{proof}
    Let $A' = USV^T$. Computing the SVD derivative for $A'$, we have 
    \begin{align*}
  d^2\frac{dU}{dt} =&   U\left( F\odot\left[ U^TA^*V\sqrt{S^2 + \sigma^4w^2}S + \sqrt{S^2 + \sigma^4w^2}U^TA^*VS \right]\right)\\
    &+ U\left( F\odot\left[ \sqrt{S^2 + \sigma^4w^2}SV^TA^*U + V^TA^*U S\sqrt{S^2 + \sigma^4w^2} \right]\right)\\
    &+ U\left(F\odot\left[O(s_Kd^{-\lambda})(U^TA^*VS + SV^TA^*U - 2S^2)  \right]\right)
\end{align*}
    Assume that $a_k\neq a_\ell$.  Notice that 
    \[
    F_{k\ell} (O(s_kd^{-\lambda})U^TA^*VS)_{k\ell} = O(d^{1-\lambda})
    \]
    Using the same technique as \ref{lem: x dynamics} we have 
    \begin{align*}
        &\frac{1}{2}d^2\frac{d}{dt}tr[U(k,k)^TU(k,k)] \\
        \geq & cdx + O(dx^{\frac{3}{2}}) + O(\sqrt{dx}) \\
        &+ tr[U(k,k)^T \sum_{j\neq k}U(k,j) (F\odot\left[O(s_Kd^{-\lambda})(U^TA^*VS + SV^TA^*U - 2S^2)  \right])]\\
        \geq & cdx + O(dx^{\frac{3}{2}}) + O(d^{1-\lambda} \sqrt{x}) + O(\sqrt{dx})
    \end{align*}
    Similar conclusion also holds for $tr[U(k,k)^TV(k,k)]$ and $tr[V(k,k)^TV(k,k)]$. We conclude that 
    \[
    d^2\frac{dx}{dt} \leq -cdx + O(dx^{\frac{3}{2}}) + O(d^{1-\lambda} \sqrt{x}) + O(\sqrt{dx}),
    \]
    and $A'_{[t_1,t_2]}\subset P_2(C_4,\min(\gamma,1,2\lambda ))$. 
\end{proof}
\begin{lemma}
\label{lem: A_theta gap}
    Assume that $A'_{[t_1,t_2]}\subset P_2(C_4,\gamma)$ and $A_{t_1}\in P_1$. Let $s_1',\ldots, s_K'$ be the first $K$ singular values of $A'$. Assume that the dynamics of $A'$ is the following.
    \[
    d^2\frac{dA'}{dt} = (A^*-A') \sqrt{A^{'T}A' + \sigma^4w^2I} + \sqrt{A'A^{'T}+\sigma^4w^2}(A^*-A') + O(d^{-\lambda}s_K')(A^*-A) + (A^*-A)O(d^{-\lambda}s_K')
    \]
    Here, $O(d^{-\lambda}s_K')$ is a matrix whose operator norm is bounded by $O(d^{-\lambda}s_K')$. Then $A_{[t_1,t_2]}\subset P_1$. 
\end{lemma}

\begin{proof}
    Proceeding as in lemma \ref{lem: long time gap}, we have 
    \[
    d^2\frac{ds_p}{dt} = 2( O(\sqrt{d}) + O(d^{1-\gamma}) + a_pd -s_p )(\sqrt{s_p^2 + \sigma^4w^2} + O(d^{-\lambda})s_K)
    \]
    if $p$ is a signal, and 
    \[
    d^2\frac{ds_{K+1}}{dt} = 2(O(\sqrt{d}) + O(d^{1-\gamma}) -s_{K+1})(\sqrt{s_{K+1}^2 + \sigma^4w^2} + O(d^{-\gamma})s_K).
    \]
    Now the conclusion follows from Gr\"onwall.
\end{proof}

\begin{theorem}
    Assume that $A_{\theta(t)}\subset P_1\cap P_2(C,\gamma)$. Then $A_{[\theta([t,T])}\subset P_1\cap P_2(C,\gamma')$ for some constant $\gamma'$. 
\end{theorem}
\begin{proof}
    We use induction lemma \ref{lem: induction lemma} to prove the theorem. For the first requirement, notice that for every $i = 1,2,\ldots, w$,
    \begin{align*}
        d^2\abs{\frac{ds_i}{dt}} \leq 4\|A^*\|_{op}^2. 
    \end{align*}
    If $A_{\theta(t)}\in P_1$, then $A_{\theta([t,t + \tau])}\subset P_1'$ for $\tau = c\frac{\sigma^2\sqrt{wd}}{d^2\|A^*\|_{op}^2}$. Here $c$ is some small constant that depends only on $P_1$. This proves the first requirement. For the second requirement, we observe that $A_{\theta(t)}$ satisfies the dynamics described in lemma E.1 at each stage, as long as $\|\hat{C}_1 - C_1\|_{op}\leq O(s_Kd^{-\lambda}$ for some $\lambda > 0$. It suffices to show that there exists some positive $\lambda$ independent of $w$ such that $\|\hat{C}_1-C_1\|_{op}\leq O( d^{-\lambda}s_K)$ in $[0,T]$. Let $T_1$ be the first time when $\|A\|_{op} = \sigma^2w$. In $[0,T_1]$, all singular values are of the same order. Therefore 
    \[
    \|\hat{C}_1 - C_1\|_{op}\leq O(\|C_1\|\sqrt{\frac{d}{w}}) = O(s_Kd^{-\frac{\gamma_{w}-1}{2}}).
    \]
    Now let $T_2 $ be the first time when $s_1 = \sigma^2wd^{\frac{\gamma_{w}-1}{4}}$. In $[T_1,T_2]$, we have 
    \[
     \|\hat{C}_1 - C_1\|_{op}\leq O(\|C_1\|\sqrt{\frac{d}{w}}) \leq O(s_Kd^{-\frac{\gamma_{w}-1}{4}})
    \]
    Therefore we can pick $\lambda = \frac{\gamma_{w}}{4}$. During $[T_1,T_2]$, we have $s_K(0) \geq c\sigma^2w$ for some constant $c$, and
    \[
    d^2\frac{ds_K}{dt} \geq s_K(\frac{1}{2}a_Kd - s_K)
    \]
    Moreover, since $d^2\frac{ds_1}{dt} \leq  (2a_1 - s_1)(s_1 +\sigma^2w) $, we have $T_2 - T_1 \geq c\frac{\log w}{w}$ for some $c$. Therefore we have $s_K(T_2) \geq \sigma^2wd^{-\nu}$ for some constant $\nu > 0$. Now in $[T_2,T]$, we have weak bound $\|\hat{C}_1 - \hat{C} \|_{op}\leq O(\sigma^2w)$. We can therefore pick $\lambda = \nu$ for $[T_2,T]$. The third requirement is verified in lemma \ref{lem: A_theta gap}. 
\end{proof}
 \subsection{First NTK Regime}
As in $A_t$, we need to show that at the end of the NTK regime, $A_{\theta(t)}$ must be in $P_1$ and $P_2$. Based on what we already have for $A_t$, this conclusion follows from Gr\"onwall. 
\begin{lemma}
    Assume that $A_0 = A_{\theta(0)}$. Let $T_1'$ be the first time when $\|A_t\|_{op} + \|A_{\theta(t)}\|_{op}$ reaches $\sigma^2w$. Then 
    \[
    \|A_{T_1'} - A_{\theta(T_1')}\|_{op}\leq O(\sigma^2\sqrt{wd}). 
    \]
    In particular, $A_{\theta(T_1')}\in P_1\cap P_2(C, \min(\frac{\gamma_{w}-1}{8}, \frac{1}{4}))$. 
\end{lemma}
\begin{proof} From dynamics of $A_t$ we see that $T_1'\leq \frac{c}{w}$ for some constant $c$. 
The dynamics of $A_t$ is the following. 
\begin{align*}
        d^2\frac{dA_t}{dt} =& (A^*-A_t)\sqrt{A^T_tA_t + \sigma^2wI} + \sqrt{A_tA_t^T + \sigma^4w^2I}(A^*-A_t).
     \end{align*}
    The dynamics of $A_{\theta(t)}$ can be written as follows. 
    \begin{align*}
        d^2\frac{dA_{\theta(t)}}{dt} = &(A^*-A_{\theta(t)})\sqrt{A_{\theta(t)}^TA_{\theta(t)} + \sigma^4w^2I} + \sqrt{A_{\theta(t)}A_{\theta(t)}^T + \sigma^4w^2I}(A^*-A_{\theta(t)}) \\
        &+ O(\sigma^2\sqrt{wd}) (A^*-A_{\theta(t)}) + (A^*-A_{\theta(t)})O(\sigma^2\sqrt{wd}) 
    \end{align*}
    Observe that 
    \begin{align*}
        \|A_{\theta(t)}^TA_{\theta(t)} - A_t^TA_t\|_{op} \leq& \|A_{\theta(t)}^T(A_{\theta(t)} - A_t)\|_{op} + \|(A_{\theta(t)} -A_t)^TA_t\|_{op} \\
        \leq & \sigma^2w\|A_{\theta(t)} - A_t\|_{op}.
    \end{align*}
    This implies that we have the following inequality on positive definite matrices. 
    \begin{align*}
        A_t^TA_t-\sigma^2w\|A_{\theta(t)}-A_t\|_{op}I \leq A_{\theta(t)}^TA_{\theta(t)} \leq A_t^TA_t + \sigma^2w\|A_{\theta(t)}-A_t\|_{op}I
    \end{align*}
    Assume that $A_t = USV^T$, then 
    \begin{align*}
       & \|\sqrt{A_{\theta(t)}^TA_{\theta(t)}+ \sigma^4w^2I} - \sqrt{A_t^TA + \sigma^4w^2I}\|_{op}\\
       \leq & \| \sqrt{\sigma^4w^2 I + A_t^TA_t + \sigma^2w\|A_{\theta(t)}-A_t\|_{op}I} - \sqrt{\sigma^4w^2 I + A_t^TA_t - \sigma^2w\|A_{\theta(t)}-A_t\|_{op}I}\|_{op}\\
       = & \|\sqrt{\sigma^4w^2I + \sigma^2w\|A_{\theta(t)}-A_t\|_{op}I + S^2} - \sqrt{\sigma^4w^2I - \sigma^2w\|A_{\theta(t)}-A_t\|_{op}I + S^2}\|_{op}\\
       =& \max_i\left(\sqrt{\sigma^4w^2 +s_i^2 +\sigma^2w\|A_{\theta(t)}-A_t\|_{op}  } - \sqrt{\sigma^4w^2 +s_i^2 -\sigma^2w\|A_{\theta(t)}-A_t\|_{op}  }\right)\\
       \leq & 3\|A_{\theta(t)} - A_t\|_{op}
    \end{align*} 
    We can control the difference between $A_t$ and $A_{\theta(t)}$. 
    \begin{align*}
        d^2\frac{d}{dt}\|A_t - A_{\theta(t)}\|_{op} \leq O(d)\|A_t - A_{\theta(t)}\| + O(d)O(\sigma^2\sqrt{wd})
    \end{align*}
    Gr\"onwall inequality implies that 
    \begin{align*}
        \|A_{T_1'} - A_{\theta(T_1')}\|_{op} 
        \leq& O(\sigma^2\sqrt{wd})
    \end{align*}
    We check that $A_{\theta(T_1')}$ satisfies the $P_1$. Since $\|A_{T_1'} - A_{\theta(T_1')}\|_{op} \leq O(\sigma^2\sqrt{wd})$, $\|A_{\theta(T_1')}\|_{op}\geq \frac{1}{3}\sigma^2w$. Let $\sigma_j(A_{T_1'})$ and $\sigma_j(A_{\theta(T_1')})$ be the $j$-th singular value of $A_{T_1'}$ and $A_{\theta(T_1')}$. If $a_j\neq a_k$, $j,k\leq K$, then $\abs{\sigma_j(A_{T_1'}) - \sigma_k(A_{T_1'})} \geq c\sigma^2w$, and therefore $$\abs{\sigma_j(A_{\theta(T_1')} )- \sigma_k(A_{\theta(T_1')})}\geq c\sigma^2w + O(\sigma^2\sqrt{wd}).$$
    The noise is clearly of order $O(\sigma^2\sqrt{wd})$. This completes the verification of $P_1$. By lemma \ref{lem: A_theta NTK Final Alignment}, 
    $\|A_{\theta(T_1')} - \Sigma \|_{op}\leq \sigma^2 wO(\sigma^2w)d^{-\min(\frac{\gamma_w - 1}{8},  \frac{1}{4}  )} $. By Lemma \ref{lem: short time gap} we are done.
    
\end{proof}
\subsection{Summary of Approximate Dynamics of $A_{\theta(t)}$ at Each Stage}
\label{subsection: dynamics summary}
In previous sections we have proved that $A_{\theta(t)}$ satisfies $P_1$ and $P_2(C,\gamma)$ for some $\gamma > 0$. The $P_1$ and $P_2(C,\gamma)$ property actually implies that the dynamics of $A_{\theta(t)}$ is such that each group of singular values evolve independently. We state the approximate dynamics for each stage of the dynamics, and show that the alignment will be improved, if the alignment is not already good enough. 
\begin{itemize}
    \item \textbf{Initialization}. At initialization, $A_{\theta(0)}$ is a random matrix. The mean of each entry is 0 and the variance of each entry is $\sigma^2\sqrt{w}$. The gap between singular values is infinitely smaller than the magnitude of singular values, and the singular vectors are not aligned with $A^*$. 
    \item \textbf{Initialization to $\|A_{\theta(t)}\|_{op} = \sigma^2w$}. Let $T_1$ be the first time when $\|A_{\theta(t)}\|_{op}$ reaches $\sigma^2w$. $[0,T_1]$ corresponds the very short NTK regime for the signals, and the dynamics of $A_{\theta(t)}$ is approximately linear. The evolution of signal singular values are roughly linear (i.e., bounded above and below by linear functions), and at $T_1$, we have $A_{\theta(T_1)}\in P_2(C,\min(\frac{\gamma_w -1}{8}, \frac{1}{4}))$. Since the singular values grows roughly linearly, $T_1 = O(d)$.

    \item \textbf{$\|A_{\theta(t)}\|_{op} = \sigma^2w$ to $\|A_{\theta(t)}\|_{op} = \sigma^2w\left(\frac{w}{d}\right)^{\frac{1}{4}}$}. Let $T_2$ be the first time when $\|A_{\theta(t)}\|_{op} = \sigma^2w\left(\frac{w}{d}\right)^{\frac{1}{4}}$. $[T_1,T_2]$ is the early stage of saddle-to-saddle dynamics. The dynamics of $s_i$ is given by 
    \[
    d^2\frac{ds_i}{dt} = 2(a_id(1+o(1)) - s_i)\sqrt{s_i^2 + \sigma^4w^2}(1+o(1)).
    \]
    By theorem 1, we have $\|\hat{C}_1 - C_1\|_{op}\leq \sigma^2w\left(\frac{w}{d}\right)^{-\frac{1}{4}}$. Let $d^{-\lambda}\sigma^w = \sigma^2w \left(\frac{w}{d}\right)^{-\frac{1}{4}}$, we have $\lambda =\frac{\gamma_w -1}{4} $. By lemma 4,1, the dynamics of $x$ satisfies 
    \[
    d^2\frac{dx}{dt} \leq -cdx + O(dx^{\frac{3}{2}}) + O(\sqrt{x}d^{1-\frac{\gamma_w -1}{4}} ) + O(\sqrt{dx})
    \]
    with $x(T_1) = \frac{\gamma_w -1}{2}$. We conclude that 
    \[
    x(T_2) \leq O(d^{-\frac{\gamma_w -1}{2}}).   
    \]
    \item  $\|A_{\theta(t)}\|_{op} = \sigma^2w\left(\frac{w}{d}\right)^{\frac{1}{4}}$ to $s_K(A_{\theta(t)}) = \frac{1}{2}a_Kd$. Let $T_3$ be the first time $s_K(A_{\theta(t)}) = \frac{1}{2}a_Kd$. At time $T_2$, we have $s_K \geq \sigma^2w d^{\delta}$ for some $\delta > 0$ that depends only on $a_1,\ldots, a_K$. Then we have 
    \[
    \|\hat{C}_1 - C_1\|_{op}\leq O(s_Kd^{-\delta}),
    \]
    \[
    d^2\frac{dx}{dt} \leq -cdx + O(dx^{\frac{3}{2}}) + O(\sqrt{x}d^{1-\delta} ) + O(\sqrt{dx}).
    \]
    We conclude that $x(T_3) = O(d^{-2\delta})$. The dynamics of $s_i$ is given by 
    \[
    d^2\frac{ds_i}{dt} = 2(a_id(1+O(d^{-2\delta})) - s_i)\sqrt{s_i^2 + \sigma^4w^2}(1+O(d^{-2\delta})).
    \]
    \[
    s_i(a_id - s_i) \leq d^2\frac{ds_i}{dt} \leq 2(a_id + Cd^{1-2\delta} - s_i)(s_i + \sigma^2w).
    \]
    The bounds on $\frac{ds_i}{dt}$ implies that 
    $$T_3 - T_1 = \frac{d\log \frac{a_Kd}{\sigma^2w}}{a_K} + O(d) = \frac{1-\gamma_{\sigma^2} - \gamma_w}{a_K}d\log d + O(d). $$
    
    \item \textbf{Final Stage}. Let $t \geq T_3$. Since $s_K = O(d)$, we have 
    \[
    \|\hat{C}_1 - C_1\|_{op} \leq O(s_K\frac{\sigma^2w}{d}) = O(s_Kd^{\gamma_{\sigma^2} + \gamma_w -1});
    \]
    \[
    d^2\frac{dx}{dt} \leq -cdx + O(dx^{\frac{3}{2}}) + O(\sqrt{x}d^{\gamma_{\sigma^2} + \gamma_w}) +O(\sqrt{dx}).
    \]
    Recall that the constant $c$ in term $-cdx$ must satisfy 
    \[
    c\leq \frac{1}{d} \min_{k,j: a_k\neq a_j}( \frac{(s_k^*-s_j^*)(s_k+s_j)(\tilde{s}_k + \tilde{s}_j)}{s_k^2-s_j^2}, \frac{(s_k-s_j)(s_k^*+s_j^*)(\tilde{s}_k + \tilde{s}_j)}{s_k^2-s_j^2} ) 
    \]
    As a result, we can take
    \begin{align*}
        c = c(a_1,\ldots,a_K)= \frac{\min_{k,j:a_k\neq a_j}\abs{a_k-a_j} a_K^2}{\max_{k,j: a_k\neq a_j} \abs{a_k^2 - a_j^2}}
    \end{align*}
    Let $c'$ be a large constant such that if $x > c'( d^{-1}+ d^{2(\gamma_{\sigma^2} + \gamma_w - 1)})$, then $\frac{dx}{dt}\leq -\frac{c(a_1,\ldots, a_K)d}{2}x$. Let $T_4$ be the first time when $x \leq c'(d^{-1}+ d^{2(\gamma_{\sigma^2} + \gamma_w - 1) } )$  after $T_3$. Then $T_4 - T_3 \leq -\frac{2d}{c(a_1,\ldots,a_K)}\log (d^{-1}+ d^{2(\gamma_{\sigma^2} + \gamma_w - 1) })+O(d)$. Moreover, for every $t > T_4$, $x$ cannot be larger than $c'(d^{-1}+ d^{2(\gamma_{\sigma^2} + \gamma_w - 1) } )$ because $\frac{dx}{dt} <0$ if $x = c'(d^{-1}+ d^{2(\gamma_{\sigma^2} + \gamma_w - 1) } )$. After $T_4$, the dynamics of each singular value is given by 
    \[
    d^2\frac{ds_i}{dt} = 2((1+O(d^{-1}+ d^{2(\gamma_{\sigma^2} + \gamma_w - 1) }))a_id - s_i)s_i(1+O(d^{2(\gamma_{\sigma^2} + \gamma_w - 1) }))
    \]
    Let $T_5$ be the first time after $T_4$ when $\abs{s_i - a_id}\leq O(d^{-1}+ d^{2(\gamma_{\sigma^2} + \gamma_w - 1)})\log d$. Then 
    $$T_5 - T_4 \leq -\frac{\max(-1,2(\gamma_{\sigma^2} + \gamma_w - 1) )}{2a_K}d\log d+ O(d\log\log d).$$
    We conclude that at time 
    \begin{align*}
            T^* =& \left(\frac{1-\gamma_{\sigma^2} - \gamma_w}{a_K} + \frac{2\max(1,2(-\gamma_{\sigma^2} - \gamma_w + 1) )}{c(a_1,\ldots,a_K)} +  \frac{\max(1,2(-\gamma_{\sigma^2} - \gamma_w + 1) )}{2a_K} \right)d\log d\\
            &+ O(d\log \log d)
    \end{align*}

    we have 
    \[
    x(T^*)\leq O(d^{\max(-1,2(\gamma_{\sigma^2} + \gamma_w - 1) )}),
    \]
    \[
    A_{\theta(T^*)}\in P_2(C ,\max\{ 1, 2(1- \gamma_{\sigma^2} - \gamma_w)\},
    \]
    and \[
    \abs{s_i(T^*)-a_id}\leq O(d^{-1}+ d^{2(\gamma_{\sigma^2} + \gamma_w - 1)})\log d
    \]
\end{itemize}

\subsection{Analysis of Testing Error}
\label{subsec: error analysis}
In section \ref{subsection: dynamics summary} we proved that $A_{\theta(t)}$ is almost aligned with $A^*$ throughout the training. With the alignment in hand, we are ready to give a time to stop training and the testing error at the end of training. 

\begin{theorem}  
    Assume that $A_{\theta(t)}$ follows the gradient flow dynamics. At time   
    \[
    T^* = \left(\frac{1-\gamma_{\sigma^2} - \gamma_w}{a_K} + \frac{2\max(1,2(-\gamma_{\sigma^2} - \gamma_w + 1) )}{c(a_1,\ldots,a_K)} +  \frac{\max(1,2(-\gamma_{\sigma^2} - \gamma_w + 1) )}{2a_K} \right)d\log d+O(d\log \log d),
    \]
    the testing error 
    \[
   \|A_{\theta(T^*)} - A^*\|_F^2 \leq O(\sigma^4wd^2) + O(\sigma^4w^2\log^2d) + O(d^{\frac{3}{2}}) +  O(\sigma^2wd)
    \]
\end{theorem}
\begin{proof}

    Let $P_K$ be the projection to the largest $K$ singular values. Then 
    \[
    \|A_{\theta(t)} - A^*\|_F^2 \leq \|P_KA_{\theta(t)} - A^*\|_F^2 + \|(I-P_K)A_{\theta(t)}\|^2_F.
    \]
    Let $s_1, \ldots, s_d$ be the singular values of $A_{\theta(t)}$. Then $\|(I-P_K)A_{\theta(t)} \|_F^2 = \sum_{p\geq K+1}s_p^2$.
    The derivative of $s_1,\ldots, s_d$ reads
    \begin{align*}
        d^2\frac{dS}{dt} =& I\odot \left( U^T(A^*+E)V\sqrt{S^2 + \sigma^4w^2}  + \sqrt{S^2 + \sigma^4w^2}U^T(A^*+E)V - 2S\sqrt{S^2 + \sigma^4w^2}\right) \\
        &+ I\odot \left( U^T(A^*+E)VO(\sigma^2w_1)  + O(\sigma^2w_1)U^T(A^*+E)V - 2SO(\sigma^2w_1)\right) 
    \end{align*}
    If $p \geq K+1$, then $s_p\leq C'\sigma^2w$ for some constant $C'$, and 
    \begin{align*}
        d^2\abs{\frac{ds_p}{dt}}\leq \abs{ (U^TA^*V)_{pp} + (U^TEV)_{pp}} O(\sigma^2w).
    \end{align*}
    \begin{align*}
       d^2 \frac{d}{dt}\sum_{p\geq K+1}s_p^2 \leq& \sum_p s_p \left( \sum_{q\leq K} U_{qp}V_{qp}A^*(qq) + (U^TA^*V)_{pp}\right)O(\sigma^2w)\\
        \leq & \sum_p s_p \left( \sum_{q\leq K} \abs{U_{qp}} \abs{V_{qp}} O(d) + O(\sqrt{d})\right)O(\sigma^2w)\\
        \leq& \sum_{q\leq K} \left(\sum_ps_p^2\right)^{\frac{1}{2}} \left(\sum_p\abs{U_{qp}}^2\right)^{\frac{1}{2}} \max_p\abs{V_{qp}} O(\sigma^2wd) + \left(\sum_ps_p^2\right)^{\frac{1}{2}}O(\sigma^2wd)\\
        \leq& \left(\sum_ps_p^2\right)^{\frac{1}{2}}\left(O(d^{1-\gamma}\sigma^2w) + O(\sigma^2wd) \right)
    \end{align*}

    We conclude that 
    \[
    d^2\frac{d}{dt}\|(I-P_K)A_{\theta(t)}\|_F \leq O(\sigma^2wd),
    \]
    which implies that 
    \[
    \|(I-P_K)A_{\theta(t)}\|_F \leq \sqrt{d}O(\sigma^2\sqrt{wd}) + O(\sigma^2wd\frac{\log d}{d}) =  O(\sigma^2\sqrt{w}d + \sigma^2w\log d)
    \]
    \[
    \|(I-P_K)A_{\theta(t)}\|_F^2 \leq O(\sigma^4wd^2 + \sigma^4w^2\log ^2d)
    \]

    Next we estimate $\|P_KA_{\theta(t)} - A^*\|_F^2 $. Notice that $\|P_KA_{\theta(t)} - A^*\|_F^2 = \sum_{i,j}\|P_KA_{\theta(t)}(i,j) -A^*(i,j)\|_F^2$. If $i\neq j$, then $A^*(i,j) = 0$, and 
    \begin{align*}
        \|P_KA_{\theta(t)}(i,j)\|_F^2\leq& \sum_{k: \text{signal},k\neq i} \|U(i,k)\|_F^2\|S(k,k)V(j,k)^T\|_F^2\\
        &+ \|U(i,j)S(j,j)\|_F^2 + \|V(j,k)^T\|_F^2\\
        \leq& O(xd^2)
    \end{align*}
    If $i = j = m+1$, we also have $A^*(m+1,m+1) = 0$, and 
    \begin{align*}
        \|P_KA_{\theta(t)}(m+1,m+1)\|_F^2 \leq \sum_{k: \text{signal}}O(d^2)\|U(m+1,k)\|_F^2 + \|V(m+1,k)\|_F^2\leq O(d^2x)
    \end{align*}
    Now assume that $i = j$ are both signals. Then 
    \begin{align*}
        \|P_KA_{\theta(t)}(i,i) - A^*(i,i)\|_F^2 \leq & \|U(i,i)(S(i,i)-S^*(i,i))V^T(i,i)\|_F^2 + O(d^2)\|U(i,i)V(i,i)^T - I\|_F^2\\
        \leq & O(d^{-1}+ d^{2(\gamma_{\sigma^2} + \gamma_w - 1)})\log d+ O(d^2\sqrt{x})
    \end{align*}
    Since there are only finitely many blocks in total, we conclude that at time $T^*$, 
    \begin{align*}
        \|A_{\theta(T^*)} - A^*\|_F^2 \leq O(\sigma^4wd^2+\sigma^4w^2\log^2d + d^2O(d^{\max(-\frac{1}{2},(\gamma_{\sigma^2} + \gamma_w - 1) )})  )+O(d^{-1}+ d^{2(\gamma_{\sigma^2} + \gamma_w - 1)})\log d
    \end{align*}
    \begin{equation}
        \|A_{\theta(T^*)} - A^*\|_F^2 \leq O(\sigma^4wd^2) + O(\sigma^4w^2\log^2d) + O(d^{\frac{3}{2}}) +  O(\sigma^2wd)
    \end{equation}

\end{proof}

\section{Gradient Descent Dynamics and Proof of Theorem \ref{thm: main theorem}}
\label{sec: GD and Proof of main thm 2}
In this section we prove that gradient descent dynamics of $A_{\theta(t)}$ is well-approximated by gradient flow dynamics of $A_{\theta(t)}$. 
\subsection{Gradient Descent vs Gradient Flow}
To study the dynamics of $A_t$ under gradient flow, we show that if the learning rate is small enough, then the gradient flow dynamics will be close to the gradient descent dynamics. 
\begin{lemma}
\label{lem: gd vs gf}
    Assume that $A$ is a matrix (not necessarily square matrix). $F$ is a function: $\mathbb{R}^{\dim A}\to \mathbb{R}^{\dim A}$. The norm $\|\cdot\|$ satisfies $\|AB\|\leq \|A\|\|B\|$ for all $A$ and $B$. In particular, operator norm and Frobenius norm satisfies this property. Assume that $\sup_A \|F(A)\|\leq C_0$ and  $\|\nabla F\|\leq C_1$ for some constant.  
    Consider gradient flow dynamics and gradient descent dynamics.
    \[
    \text{Gradient Flow: } \frac{dA_{f}}{dt} = F(A_f)
    \]
    \[
    \text{Gradient Descent: } \frac{A_{d}((k+1)\eta) - A_{d}(k\eta)}{\eta} = F(A_d(k\eta))
    \]
    Assume that $A_f(0) = A_d(0)$. Then 
    \[
    \|A_f - A_d\|(k\eta) \leq ((1+\eta C_1)^{k-1}-1) \frac{1}{2}\eta C_0
    \]
\end{lemma}
\begin{proof}
    Notice that we have 
    \[
    A_f((k+1)\eta) - A_f(k\eta) = \int_0^\eta F(A_f(k\eta + t))dt
    \]
    \[
    (A_f - A_d)((k+1)\eta) - ((A_f - A_d)(k\eta)) = \int_0^\eta F(A_f(k\eta + t))- F(A_d(k\eta))dt
    \]
    \begin{align*}
        \int_0^\eta F(A_f(k\eta + t))- F(A_d(k\eta))dt =& \int_0^\eta F(A_f(k\eta + t))- F(A_f(k\eta))dt \\
        &+ \eta (F(A_f(k\eta ))- F(A_d(k\eta)))\\
    \end{align*}
    Let $G(t) = \int_0^t F(A_f(k\eta + s))ds$. Then 
    \begin{align*}
        \int_0^\eta F(A_f(k\eta + t))- F(A_f(k\eta))dt =& G(\eta) - G(0) - \eta G'(0) - \frac{1}{2} \eta^2 G''(\xi)\\
        =& \frac{1}{2}\eta^2 \frac{d}{dt}|_{t = \xi} F(A_f(k\eta + t))\\
        =& \frac{1}{2} \eta^2 \frac{\partial F}{\partial A}(A_f(k\eta + \xi))\frac{d}{dt}|_{t = \xi} A_f(k\eta + t)\\
        =& \frac{1}{2}\eta^2 \frac{\partial F}{\partial A} (A_f(k\eta + \xi)) F(A_f(k\eta + \xi))
    \end{align*}
    \[
    \|\int_0^\eta F(A_f(k\eta + t))- F(A_f(k\eta))dt\| \leq \frac{1}{2}\eta^2 C_0C_1
    \]
    \[
    \eta \|(F(A_f(k\eta ))- F(A_d(k\eta)))\| \leq \eta\|\nabla F\|\|A_f(k\eta) - A_d(k\eta)\|\leq \eta C_1\|A_f(k\eta ) - A_d(k\eta)\|
    \]
    We conclude that 
    \begin{align*}
        \| A_f((k+1)\eta) - A_d((k+1)\eta)\| \leq& (1+\eta C_1)\|A_f(k\eta) -A_d(k\eta)\| + \frac{1}{2}\eta^2 C_0C_1
    \end{align*}
    \begin{align*}
        \| A_f - A_d\|((k+1)\eta) + \frac{1}{2}\eta C_0 \leq& (1+\eta C_1)(\|A_f -A_d\|(k\eta) + \frac{1}{2}\eta C_0)\\
        \leq& (1+\eta C_1)^{k} \frac{1}{2}\eta C_0
    \end{align*}
    \[
    \| A_f - A_d\|(k\eta) \leq ((1+\eta C_1)^{k-1}-1) \frac{1}{2}\eta C_0
    \]
\end{proof}

To apply the lemma above we need to prove that $\|A_{\theta(t)}\|_F^2 \leq O(d^2)$ throughout the training. 

\begin{lemma}
    For both lazy and active regime, we always have $\|A_{\theta(t)}\|_F^2 \leq 10( \|A^*\|_F^2 + \|E\|_F^2)$ throughout the training. 
\end{lemma}
\begin{proof}
     The gradient descent dynamics of $A_{\theta(t)}$ is given by 
    \[
    A_{\theta(t+1) } - A_{\theta(t)} = \eta d^{-2} (A^*+E-A_{\theta(t)})C_1 + C_2(A^*+E-A_{\theta(t)}).
    \]
    \begin{align*}
    &tr( (A^*+E-A_{\theta(t+1)})^T(A^* +E- A_{\theta(t+1)}) -  (A^*+E-A_{\theta(t)})^T(A^*+E - A_{\theta(t)}) )\\
    =& -2tr(  (A^*+E-A_{\theta(t)}))^T( A_{\theta(t+1)} - A_{\theta(t)} )  ) + \|A_{\theta(t+1)} - A_{\theta(t)}\|_F^2\\
    \leq & -2 \eta d^{-2}tr( (A^*+E-A_{\theta(t)})^T(C_1 + C_2)(A^*+E-A_{\theta(t)}) ) + \eta^2O(d^{-2}\|C_1\|_{op}^2)
    \end{align*}
    
    From theorem 2, we know that $C_1 + C_2\geq \frac{1}{3} \sigma^2wI$. Therefore $tr( (A^*+E-A_{\theta(t)})^T(C_1 + C_2)(A^*+E-A_{\theta(t)}) ) \geq \|A^*+E-A_{\theta(t)}\|_F^2\frac{1}{3}\sigma^2w \geq cd^2\sigma^2w$ for some constant $c$. Therefore for the lazy regime, we always have $\|A^*+E-A_{\theta(t)}\|_F^2(t+1) \leq \|A^*+E-A_{\theta(t)}\|_F^2(t) $ if $ 5( \|A^*\|_F^2 + \|E\|_F^2) \leq \|A_{\theta(t)}\|_F^2 \leq 7( \|A^*\|_F^2 + \|E\|_F^2)$, which implies that $\|A_{\theta(t)}\|_F^2 \leq 10 (\|A^*\|_F^2 + \|E\|_F^2)$ for all time. For the active regime, we have $$\|A^*+E-A_{\theta(t)}\|_F^2(t+1) -\|A^*+E-A_{\theta(t)}\|_F^2(t) \leq \eta^2O(1).$$
Since the training has at most $O(\frac{T^*}{\eta})$ steps, we see that $\forall t $, $\|A^*+E-A_{\theta(t)}\|_F^2\leq 2(\|A^*\|_F^2 + \|E\|_{F}^2)+ O(\eta T^*) \leq  10 (\|A^*\|_{F}^2 + \|E\|_F^2)$. 
\end{proof}

\begin{proof}
    [Proof of main theorem] We first consider the active regime. It suffices to consider the error from considering gradient descent, rather than gradient flow. To apply lemma \ref{lem: gd vs gf}, it is more convenient to consider the dynamics for $W_1$ and $W_2$. By lemma G.2, $\|W_1^TW_1\|_F^2 + \|W_2W_2^T\|_F^2\leq O(d^2)$ and $\|W_1\|_{op} + \|W_2\|_{op}\leq O(\sqrt{d})$. The gradient flow dynamics of $[W_1,W_2^T]$ is given by 
    \[
    \frac{d}{dt} [W_1,W_2^T] = d^{-2} [ W_2^T(A^*-W_2W_1), W_1(A^*-W_2W_1)^T].
    \]
    Let $F([W_1,W_2^T]) =  d^{-2} [ W_2^T(A^*-W_2W_1), W_1(A^*-W_2W_1)^T].$ Then $\sup_{t}\|F(W_1,W_2^T)\|_{F} \leq  O(d^{-\frac{1}{2}}) $.  Computing the differential of $F$, we obtain that 
    \begin{align*}
        dF( [W_1,W_2^T] ) = d^{-2}[ &-dW_2^T(A^*-W_2W_1) + W_2^T(-dW_2W_1 - W_2dW_1), \\
                              &-dW_1(A^*-W_2W_1)^T + W_1(-dW_1^TW_2 - W_1^TdW_2)]
    \end{align*}    
and therefore $\|\nabla F\|_{F}\leq O(d^{-1})$. In the active regime, the total number of training steps is $\eta^{-1} O(d\log d)$ By lemma \ref{lem: gd vs gf}, 
\[
\|W_1^{flow} - W_1^{descent}\|_F( T^*) \leq O(\eta d^{-\frac{1}{2}}) ( (1+\eta O(d^{-1}))^{\frac{O(d\log d)}{\eta}} -1) = O(\eta d^{-\frac{1}{2}}\log d)
\]
and the same holds true for $W_2$. We conclude that 
\begin{align*}
    &\|W_2^{flow}W_1^{flow}- W_2^{descent} W_1^{descent}\|_F \\
    \leq &(\|W_2^{descent}\|_{op} + \|W_1^{descent}\|_{op}) O(\eta d^{-\frac{1}{2}}\log d) \\
    = &O(\eta \log d)
\end{align*}
  \[
  \|A_{\theta}^{flow} - A_{\theta}^{descent}\|_F^2\leq O(\eta^2\log ^2d).
  \]  
  In the lazy regime, from strong bound we have 
  \[
  W_1^TW_1 = (1+O(\sqrt{\frac{d}{w}}))\sqrt{\sigma^4w^2I + A^TA}.
  \]
  Therefore $\|W_1\|_{op} \leq  O(\sigma\sqrt{w})$, $\sup_t\|F([W_1,W_2^T])\|_{F} \leq d^{-2}(\|W_1\|_{op} + \|W_2\|_{op})O(d) = O(d^{-1}\sigma\sqrt{w})$. Similarly, $\|\nabla F\|_F \leq d^{-2}O(\|A^*-W_2W_1\|_{F} + \|W_2\|_F\|W_1\|_F) = d^{-2} O(\sigma^2w)$. By lemma \ref{lem: gd vs gf}, at time $ \frac{100d^2\log d}{\sigma^2w}$, 
  \[
  \|W_1^{flow} - W_1^{descent} \|_F(\frac{100d^2\log d}{\sigma^2w}) \leq O(\eta d^{-1}\sigma\sqrt{w})(( 1+ \eta O(d^{-2}\sigma^2w ) )^{\frac{100d^2\log d}{\sigma^2w\eta} }-1 ) = O(\eta d^{-1}\sigma \sqrt{w}\log d).
  \]
  
  We conclude that 
\begin{align*}
    &\|W_2^{flow}W_1^{flow}- W_2^{descent} W_1^{descent}\|_F \\
    \leq &(\|W_2^{descent}\|_{op} + \|W_1^{descent}\|_{op}) O(\eta d^{-1}\sigma \sqrt{w}\log d) \\
    = &O(\eta d^{-1}\sigma^2w \log d).
\end{align*}
\[
\|W_2^{flow}W_1^{flow}- W_2^{descent} W_1^{descent}\|_F \leq O(\eta^2 d^{-2} \sigma^2w \log ^2 d).
\]
  
  Recall that in lemma D.1 we proved that 
  \[
  \|A_{\theta(t)}^{flow} - B_t\|_{F}^2 = o(d^2),   
  \]
  and at time $100\frac{d^2\log d}{\sigma^2w}$, $\|B_t- A^*-E\|_F^2\leq d^{-50}$, which implies that $\|A_{\theta}^{descent} - A^*-E\|_F^2 \leq o(d^2)$. Before this time, we have $\|B_t - A^*\|_{F}^2\geq \frac{1}{3}\min (\|A^*\|_F^2 , \|E\|_F^2) . $ After this time, we have 
  \begin{align*}
    &tr( (A^*+E-A_{\theta(t+1)}^{descent})^T(A^* +E- A_{\theta(t+1)}^{descent}) -  (A^*+E-A_{\theta(t)}^{descent})^T(A^*+E - A_{\theta(t)}^{descent}) )\\
    =& -2tr(  (A^*+E-A_{\theta(t)}^{descent}))^T( A_{\theta(t+1)}^{descent} - A_{\theta(t)} ^{descent})  ) + \|A_{\theta(t+1)}^{descent} - A_{\theta(t)}^{descent}\|_F^2\\
    \leq & -2 \eta d^{-2}tr( (A^*+E-A_{\theta(t)}^{descent})^T(C_1 + C_2)(A^*+E-A_{\theta(t)}^{descent}) ) \\
    &+ \eta^2tr ((A^*+E-A_{\theta(t)}^{descent})^T(10\sigma^4w^2I)(A^*+E-A_{\theta(t)}^{descent}))\\
    \leq & - (2\eta d^{-2}\sigma^2 - 10\eta^2\sigma^4w^2) \|A^*+E-A_{\theta(t)}^{descent}\|_F^2
    \end{align*}

Therefore $\|A^*+E-A_{\theta(t)}^{descent}\|_{F}^2$ is decreasing and therefore $\|A_{\theta}^{descent} - A^*\|_{F}^2\geq \frac{1}{3}\min(\|A^*\|_F^2 , \|E\|_F^2)$ for all time.

\end{proof}

\section{Experimental setup}
\label{sec:exp_details}
We now describe the experimental setup for the experiments shown in Figures \ref{fig:mixed} and \ref{fig:heatmaps}. For all the experiments, we used the losses 
$$
\mathcal{L_\text{train}(\theta)} = \frac{1}{d^2} \left\Vert A_\theta - (A^\star + E)\right\Vert^2_F; \quad
\mathcal{L_\text{test}(\theta)} = \frac{1}{d^2} \left\Vert A_\theta - A^\star\right\Vert^2_F
$$
where $E$ has i.i.d. $\mathcal{N}(0, 1)$ entries, $A^\star = K^{-1/2}\sum_{i=1}^K u_i v_i^T$ with $u_i,v_i\sim \mathcal{N}(0, \text{Id}_d)$ Gaussian vectors in $\mathbb{R}^d.$ This means that $\text{Rank} A^\star = K.$ The factor $K^{-1/2}$ ensures that $\Vert A^\star \Vert_F = \Theta(d).$

We then either run the self-consistent dynamics (equation \eqref{eq:self-consistent}) or gradient descent (equation \eqref{eq:gd}). Following Theorem \ref{thm: main theorem}, we take a learning rate $\eta = \frac{d^2}{c w \sigma^2}$ for $\gamma_{\sigma^2} + \gamma_2 > 1,$ and $\eta = \frac{d^2}{c \Vert A^\star \Vert_\text{op}}$ otherwise, where $c$ is usually $50$ but can be taken to be $2$ or $5$ for faster convergence at the cost of more unstable training.

For the experiments in Figure \ref{fig:mixed}, we took $d=500$ and $K=5.$ For the experiments in Figure \ref{fig:heatmaps}, we took $d=200$ and $K=5.$ For making the contour plot, we took a grid with $35$ points for $\gamma_{\sigma^2} \in [-3.0, 0.0]$ and $35$ points for $\gamma_w\in [0, 2.8].$ For each of the $35^2$ pair of values for $(\gamma_{\sigma^2}, \gamma_w),$ we ran gradient descent  (and for the lower right plot the self-consistent dynamics too) until the train error converged. For all the runs, we took the same realizations of $A^\star$ and $E.$ 

All the experiments were implemented in PyTorch \cite{paszke2019pytorch}. Experiments took 12 hours of compute, using two GeForce RTX 2080 Ti (11GB memory) and two TITAN V (12GB memory).

\end{document}